\documentclass{article}
\usepackage{iclr2024_conference,times}
\usepackage{hyperref}
\usepackage{natbib}

\newcommand{\lbopt}{\underline{\vz}^*}
\newcommand{\ubopt}{\overline{\vz}^*}
\newcommand{\lbbox}{\underline{\vz}^\dagger}
\newcommand{\ubbox}{\overline{\vz}^\dagger}

\usepackage{amsmath,amsfonts,bm}

\def\eqref#1{equation~\ref{#1}}

\def\1{\bm{1}}

\def\eps{{\epsilon}}

\def\va{{\bm{a}}}
\def\vb{{\bm{b}}}

\def\vd{{\bm{d}}}

\def\vf{{\bm{f}}}

\def\vh{{\bm{h}}}

\def\vx{{\bm{x}}}
\def\vy{{\bm{y}}}
\def\vz{{\bm{z}}}

\def\mW{{\bm{W}}}

\DeclareMathAlphabet{\mathsfit}{\encodingdefault}{\sfdefault}{m}{sl}
\SetMathAlphabet{\mathsfit}{bold}{\encodingdefault}{\sfdefault}{bx}{n}

\newcommand{\E}{\mathbb{E}}

\newcommand{\R}{\mathbb{R}}

\newcommand{\Var}{\mathrm{Var}}

\DeclareMathOperator*{\argmax}{arg\,max}
\DeclareMathOperator*{\argmin}{arg\,min}

\usepackage[utf8]{inputenc} %
\usepackage{lipsum} %

\usepackage[T1]{fontenc}

\usepackage{etoolbox}
\newbool{includeappendix}
\setbool{includeappendix}{true}

\ifdefined\isoverfull
\else
\fi

\usepackage{subcaption}

\usepackage{xcolor} %

\definecolor{my-full-blue}{HTML}{1F77B4}

\definecolor{my-full-orange}{HTML}{FF7F0E}

\definecolor{my-full-green}{HTML}{2CA02C}

\definecolor{my-full-red}{HTML}{d62728}

\definecolor{my-full-purple}{HTML}{9467bd}

\definecolor{my-full-brown}{HTML}{8c564b}

\definecolor{my-full-pink}{HTML}{e377c2}

\definecolor{my-full-gray}{HTML}{7f7f7f}

\definecolor{my-full-olive}{HTML}{bcbd22}

\definecolor{my-full-cyan}{HTML}{17becf}

\definecolor{c1}{RGB}{86, 100, 26}
\definecolor{c2}{RGB}{192, 175, 251}
\definecolor{c3}{RGB}{230, 161, 118}
\definecolor{c4}{RGB}{0, 103, 138}
\definecolor{c5}{RGB}{152, 68, 100}
\definecolor{c6}{RGB}{94, 204, 171}
\definecolor{c7}{RGB}{205, 205, 205}

\definecolor{cm1}{HTML}{1f77b4}
\definecolor{cm2}{HTML}{ff7f0e}
\definecolor{cm3}{HTML}{2ca02c}
\definecolor{cm4}{HTML}{d62728}
\definecolor{cm5}{HTML}{9467bd}
\definecolor{cm6}{HTML}{8c564b}
\definecolor{cm7}{HTML}{e377c2}
\definecolor{cm8}{HTML}{7f7f7f}
\definecolor{cm9}{HTML}{bcbd22}
\definecolor{cm10}{HTML}{17becf}

\colorlet{my-blue}{my-full-blue!60}
\colorlet{my-orange}{my-full-orange!30}
\colorlet{my-green}{my-full-green!30}
\colorlet{my-red}{my-full-red!30}
\colorlet{my-purple}{my-full-purple!30}
\colorlet{my-brown}{my-full-brown!30}
\colorlet{my-pink}{my-full-pink!30}
\colorlet{my-gray}{my-full-gray!30}
\colorlet{my-olive}{my-full-olive!30}
\colorlet{my-cyan}{my-full-cyan!30}

\usepackage{listings}

\usepackage{textcomp}

\usepackage{xcolor}

\usepackage[scaled=0.8]{beramono}

\definecolor{ckeyword}{HTML}{7F0055}
\definecolor{ccomment}{HTML}{3F7F5F}
\definecolor{cstring}{HTML}{2A0099}

\lstdefinestyle{numbers}{
	numbers=left,
	framexleftmargin=20pt,
	numberstyle=\tiny,
	firstnumber=auto,
	numbersep=1em,
	xleftmargin=2em
}

\lstdefinestyle{layout}{
	frame=none,
	captionpos=b,
}

\lstdefinestyle{comment-style}{
	morecomment=[l]//,
	morecomment=[s]{/*}{*/},
	commentstyle={\color{ccomment}\itshape},
}

\lstdefinestyle{string-style}{
	morestring=[b]",%
	morestring=[b]',%
	stringstyle={\color{cstring}},
	showstringspaces=false,%
}

\lstdefinestyle{keyword-style}{
	keywordstyle={\ttfamily\bfseries},
	morekeywords={
		function,
		constructor,
		int,
		bool,
		return,
		returns,
		uint
	},
	morekeywords = [2]{},
	keywordstyle = [2]{\text},
	sensitive=true,
}

\lstdefinestyle{input-encoding}{
	inputencoding=utf8,
	extendedchars=true,
	literate=
	{ℝ}{$\reals$}1%
	{→}{$\rightarrow$}1%
	{α}{$\alpha$}1%
	{β}{$\beta$}1%
	{λ}{$\lambda$}1%
	{θ}{$\theta$}1%
	{ϕ}{$\phi$}1%
}

\lstdefinestyle{escaping}{
	moredelim={**[is][\color{blue}]{\%}{\%}},
	escapechar=|,
	mathescape=true
}

\lstdefinestyle{default-style}{
	basicstyle=\fontencoding{T1}\ttfamily\footnotesize,
	style=numbers,
	style=layout,
	style=comment-style,
	style=string-style,
	style=keyword-style,
	style=input-encoding,
	style=escaping,
	tabsize=2,
	upquote=true
}

\lstdefinelanguage{BASIC}{
	language=C++,
	style=default-style
}[keywords,comments,strings]%

\usepackage[capitalize,noabbrev]{cleveref}

\crefname{listing}{Lst.}{listings}
\crefname{line}{Lin.}{Lin.}
\crefname{appendix}{App.}{App.}
\crefname{lemma}{Lemma}{Lemmas}
\Crefname{lemma}{Lemma}{Lemmas}
\crefname{thm}{Theorem}{Theorems}
\Crefname{thm}{Theorem}{Theorems}

\newcommand{\app}[1]{%
	\ifbool{includeappendix}{\cref{#1}}{the appendix}%
}
\newcommand{\App}[1]{%
	\ifbool{includeappendix}{\cref{#1}}{The appendix}%
}

\usepackage{tikz}

\usetikzlibrary{calc,decorations,decorations.pathmorphing}
\usetikzlibrary{positioning,fit,arrows}
\usetikzlibrary{decorations.markings}
\usetikzlibrary{shapes,shapes.geometric}
\usetikzlibrary{shadows,patterns,snakes}
\usetikzlibrary{backgrounds,decorations.pathreplacing,automata}
\usetikzlibrary{intersections}
\usetikzlibrary{angles,quotes}
\usetikzlibrary{plotmarks}
\usetikzlibrary{patterns}
\usetikzlibrary{arrows.meta}
\usetikzlibrary{shapes.misc}
\usetikzlibrary{chains}
\usetikzlibrary{calc}

\usepackage{microtype}
\usepackage{graphicx}
\usepackage{booktabs} %
\usepackage{adjustbox}
\usepackage{threeparttable}
\usepackage{nameref}
\usepackage{multirow}
\usepackage{amssymb}
\usepackage{mathtools}
\usepackage{amsthm}
\usepackage[abbreviations]{foreign}
\usepackage{enumitem}
\usepackage{wrapfig}
\usepackage{accents}
\usepackage{dsfont}
\usepackage{thmtools}
\usepackage{thm-restate}
\usepackage{dcolumn}        %
\newcolumntype{C}{D{.}{.}{2.1}}

\declaretheoremstyle[
  spaceabove=0.5em, %
  spacebelow=0.01em,%
  headfont=\normalfont\bfseries,
  bodyfont=\normalfont\itshape,
]{mystyle}

\declaretheorem[name=Theorem,numberwithin=section,style=mystyle]{thm}
\crefname{thm}{Theorem}{Theorems}

\crefname{lem}{Lemma}{Lemmas}

\declaretheorem[name=Definition,numberlike=thm,style=mystyle]{mydef}
\crefname{mydef}{Definition}{Definitions}

\crefname{cor}{Corollary}{Corollaries}

\DeclareMathOperator*{\ibpm}{Box}
\newcommand{\ibpb}{\ensuremath{\ibpm^*}\xspace}
\newcommand{\ibpl}{\ensuremath{\ibpm^\dagger}\xspace}

\newcommand{\optW}{\ensuremath{\mW^*}\xspace}
\newcommand{\ibpW}{\ensuremath{\mW^\dagger}\xspace}

\DeclareMathOperator*{\diag}{diag}
\DeclareMathOperator*{\relu}{ReLU}

\newcommand{\bc}[1]{\mathcal{#1}}

\newcommand{\bs}[1]{\boldsymbol{#1}}
\newcommand{\B}{\bc{B}}
\renewcommand{\L}{\bc{L}}

\renewcommand{\paragraph}[1]{\textbf{#1}\hspace{1em}}

\newcommand{\mnbab}{\textsc{MN-BaB}\xspace}
\newcommand{\pgd}{\textsc{PGD}\xspace}
\newcommand{\milp}{\textsc{MILP}\xspace}

\newcommand{\crownibp}{\textsc{CROWN-IBP}\xspace}

\newcommand{\ibp}{\textsc{IBP}\xspace}
\newcommand{\sabr}{\textsc{SABR}\xspace}
\newcommand{\taps}{\textsc{TAPS}\xspace}
\newcommand{\staps}{\textsc{STAPS}\xspace}

\newcommand{\ibpr}{\textsc{IBP-R}\xspace}
\newcommand{\ccibp}{\textsc{CC-IBP}\xspace}
\newcommand{\mtlibp}{\textsc{MTL-IBP}\xspace}

\newcommand{\colt}{\textsc{COLT}\xspace}

\newcommand{\boxd}{\textsc{Box}\xspace}

\newcommand{\cifar}{CIFAR-10\xspace}

\newcommand{\mnist}{\textsc{MNIST}\xspace}

\newcommand{\cnnt}{\texttt{CNN3}\xspace}
\newcommand{\cnns}{\texttt{CNN7}\xspace}

\renewcommand{\th}{\textsuperscript{th}\xspace}

\newcommand\ubar[1]{\underaccent{\bar}{#1}}
\newcommand\obar[1]{\bar{#1}}

\usepackage{siunitx}
\newcolumntype{d}[1]{S[table-format=#1]}

\colorlet{cbackground}{c7!20}
\colorlet{cexact}{my-full-blue!45}
\colorlet{cexactlatent}{my-full-green!40}
\colorlet{cfwd}{black!100}
\colorlet{cpgd}{my-full-purple!75}
\colorlet{cpgdsignle}{my-full-red!90!black!65}

\colorlet{cbwd}{my-full-red!90!black!65}
\colorlet{cibp}{black!80}
\colorlet{ctool}{c4!100}
\colorlet{netinside}{c7!100}

\tikzstyle{toolstyle}=[dotted, line width = 1.2pt,draw=ctool]
\tikzstyle{ibpstyle}=[dash pattern=on 5pt off 2pt, cibp]
\tikzstyle{exactstyle}=[fill=cexact, opacity=1.0, draw=none]
\tikzstyle{exactlatentstyle}=[fill=cexactlatent, opacity=0.8, draw=none]
\tikzstyle{point}=[line width=0.5pt, draw=black, cross out, inner sep=0pt, minimum width=3pt, minimum height=3pt, anchor=center]
\tikzstyle{wcpoint}=[line width=0.5pt, draw=black, fill=black, circle, inner sep=0pt, minimum width=2.5pt, minimum height=2.5pt, anchor=center]
\tikzstyle{pgdarrow}=[color=cpgd, thick]
\tikzstyle{pgdsinglearrow}=[color=cpgdsignle, thick]
\tikzstyle{fwdarrow}=[color=cfwd, thick, dashed]
\tikzstyle{bwdarrow}=[color=cbwd,  line width = 1.2pt,, dotted]
\tikzstyle{pane}=[fill=cbackground, rectangle, rounded corners=2pt]

\newcommand{\markerb}[1]{\tikz[]{\node[fill, aspect=1, color=#1, inner sep=0pt, minimum size=2.1mm]{};}\hspace{-0.225em}\xspace}
\newcommand{\markeribp}{\protecting{\tikz[]{\node[aspect=1, draw=cibp, inner sep=0pt, minimum size=2.1mm, ibpstyle, dash pattern=on 2.5pt off 1pt]{};}\xspace}}
\newcommand{\markeropt}{\protecting{\markerb{my-full-green!50}\xspace}}
\newcommand{\markerexact}{\protecting{\markerb{my-full-blue!50}\xspace}}

\title{Understanding Certified Training\\with Interval Bound Propagation}

\author{Yuhao Mao, Mark Niklas Müller, Marc Fischer \& Martin Vechev \\
Department of Computer Science, ETH Zürich, Swizterland \\
\texttt{\{yuhao.mao, mark.mueller, marc.fischer, martin.vechev\}@inf.ethz.ch}
}

\iclrfinalcopy

\begin{document}
\maketitle
\begin{abstract}

As robustness verification methods are becoming more precise, training certifiably robust neural networks is becoming ever more relevant. To this end, certified training methods compute and then optimize an upper bound on the worst-case loss over a robustness specification. Curiously, training methods based on the imprecise interval bound propagation (IBP) consistently outperform those leveraging more precise bounds. Still, we lack a theoretical understanding of the mechanisms making IBP so successful. In this work, we investigate these mechanisms by leveraging a novel metric measuring the tightness of IBP bounds. We first show theoretically that, for deep linear models (DLNs), tightness decreases with width and depth at initialization, but improves with IBP training. We, then, derive sufficient and necessary conditions on weight matrices for IBP bounds to become exact and demonstrate that these impose strong regularization, providing an explanation for the observed robustness-accuracy trade-off. Finally, we show how these results on DLNs transfer to ReLU networks, before conducting an extensive empirical study, (i) confirming this transferability and yielding state-of-the-art certified accuracy, (ii) finding that while all IBP-based training methods lead to high tightness, this increase is dominated by the size of the propagated input regions rather than the robustness specification, and finally (iii) observing that non-IBP-based methods do not increase tightness. Together, these results help explain the success of recent certified training methods and may guide the development of new ones.

\end{abstract}

\newlength{\savedtextfloatsep}
\newlength{\savedintextsep}
\setlength{\savedtextfloatsep}{\textfloatsep}
\setlength{\savedintextsep}{\intextsep}
\setlength{\textfloatsep}{8pt plus 2pt minus 2pt}
\setlength{\intextsep}{6pt plus 1pt minus 1pt}
\section{Introduction}

Deploying deep learning in safety-critical domains increases the importance of formal robustness guarantees against adversarial examples \citep{BiggioCMNSLGR13,SzegedyZSBEGF13}.
Progress in neural network certification \citep{ZhangWXLLJ22,FerrariMJV22} has shifted attention toward training methods that improve certifiable robustness while minimizing the loss of standard accuracy.

\author{Antiquus S.~Hippocampus, Natalia Cerebro \& Amelie P. Amygdale \thanks{ Use footnote for providing further information
about author (webpage, alternative address)---\emph{not} for acknowledging
funding agencies.  Funding acknowledgements go at the end of the paper.} \\
Department of Computer Science\\
Cranberry-Lemon University\\
Pittsburgh, PA 15213, USA \\
\texttt{\{hippo,brain,jen\}@cs.cranberry-lemon.edu} \\
\And
Ji Q. Ren \& Yevgeny LeNet \\
Department of Computational Neuroscience \\
University of the Witwatersrand \\
Joburg, South Africa \\
\texttt{\{robot,net\}@wits.ac.za} \\
\AND
Coauthor \\
Affiliation \\
Address \\
\texttt{email}
}

\paragraph{Certified Training}
These certified training methods aim to compute and then optimize approximations of the network's worst-case loss over an input region defined by an adversary specification. 
To this end, they compute an over-approximation of the network's reachable set using symbolic bound propagation methods \citep{SinghGMPV18,SinghGPV19,GowalIBP2018}. Surprisingly, training methods based on the least precise bounds, obtained via interval bound propagation (\ibp), empirically yield the best performance \citep{ShiWZYH21}. \citet{jovanovic2022paradox} investigate this surprising observation theoretically and find that more precise bounding methods induce harder optimization problems. 

As a result, \emph{all} methods obtaining state-of-the-art performance leverage IBP bounds either directly \citep{ShiWZYH21}, as regularizer \citep{PalmaIBPR22}, or to precisely but unsoundly approximate the worst-case loss \citep{MuellerEFV22,MaoMFV2023,PalmaBDKSL23}. 
However, while \ibp is crucial to their success, none of these works develop a theoretical understanding of what makes \ibp training so effective and how it affects bound tightness and network regularization.

\paragraph{This Work}
To understand \ibp training's success, we derive necessary and sufficient weight conditions for tight \ibp bounds, a property we call \emph{propagation invariance}. We prove that it implies extreme regularization, consistent with the observed trade-off between certifiable robustness and accuracy \citep{TsiprasSETM19,MuellerEFV22}.
We measure proximity to propagation invariance through \emph{propagation tightness}, the ratio of optimal and \ibp bounds, and show how to compute it efficiently globally for deep linear networks (DLNs) and locally for ReLU networks.

This novel metric enables us to theoretically investigate the effects of model architecture, weight initialization, and training methods on \ibp bound tightness for deep linear networks (DLNs). We show that (i) at initialization, tightness decreases with width (polynomially) and depth (exponentially), (ii) tightness is increased by \ibp training, and (iii) sufficient width becomes crucial for trained networks. 

Conducting an extensive empirical study, we confirm the predictiveness of our theoretical results for deep ReLU networks and observe that:
(i) increasing network width but not depth improves state-of-the-art certified accuracy, (ii) \ibp training significantly increases tightness, almost to the point of propagation invariance, (iii) unsound \ibp-based training methods increase tightness to a smaller degree, determined by the size of the propagated input region and the weight of the \ibp-loss, but yield better performance, and (iv) non-\ibp-based training methods barely increase tightness, leading to higher accuracy but worse robustness. 
These findings suggest that \ibp-based training improves robustness by increasing tightness at the cost of standard accuracy, although high tightness is not generally necessary for robustness. This explains the success of unsound \ibp-based methods and motivates more effective certified training methods.

\section{Background} \label{sec:background}
Here, we provide a background on adversarial and certified robustness. We consider a classifer $\vf\colon \R^{d_\text{in}} \mapsto \R^{c}$ predicting a numerical score $\vy \coloneqq \vf(\vx)$ per class given an input $\vx \in \bc{X} \subseteq \R^{d_\text{in}}$.

\paragraph{Adversarial Robustness} describes the property of a classifer $\vf$ to consistently predict the target class $t$ for all perturbed inputs $\vx'$ in an $\ell_p$-norm ball $\bc{B}_p^{\epsilon_p}(\vx)$ of radius $\epsilon_p$. As we focus on $\ell_\infty$ perturbations in this work, we henceforth drop the subscript $p$ for notational clarity. More formally, we define \emph{adversarial robustness} as:
\begin{equation}
	\label{eq:adv_robustness}
	\argmax_j f(\vx')_j = \{t\}, \quad \forall \vx' \in \bc{B}_p^{\epsilon_p}(\vx) := \{\vx' \in \bc{X} \mid \|\vx -\vx'\|_p \leq \epsilon_p\}.
\end{equation}

\paragraph{Neural Network Certification} can be used to formally prove the robustness of a classifier $\vf$ for a given input region $\bc{B}^{\epsilon}(\vx)$.
Interval bound propagation (IBP) \citep{GowalIBP2018,MirmanGV18} is a simple, popular certification method. It propagates $\bc{B}^{\epsilon}(\vx)$ through the network using \boxd over-approximations, one interval per dimension, after each layer. Certification checks whether every point in the output over-approximation yields the correct classification.
As an example, consider an $L$-layer network $\vf = \vh_L \circ \bs{\sigma} \circ \vh_{L-2} \circ \dotsc \circ \vh_1$, with linear layers $\vh_i$ and ReLU activation functions $\bs{\sigma}$.
We first over-approximate the input region $\bc{B}^{\epsilon}(\vx)$ as \boxd with radius $\bs{\delta}^0 = \epsilon$ and center $\dot{\vx}^0 = \vx$, such that we have the $i$\th dimension of the input $x^0_i \in [\ubar{x}_i, \obar{x}_i] \coloneqq [\dot{x}^0_i - \delta^0_i, \dot{x}^0_i + \delta^0_i]$.
Propagating such a \boxd through the linear layer $\vh_i(\vx^{i-1}) = \mW \vx^{i-1} + \vb =: \vx^i$, we obtain the output hyperbox with centre $\dot{\vx}^i = \mW \dot{\vx}^{i-1} + \vb$ and radius $\bs{\delta}^{i} = |\mW| \bs{\delta}^{i-1}$, where $| \cdot |$ denotes the element-wise absolute value.
To propagate a \boxd through the ReLU activation $\relu(\vx^{i-1}) \coloneqq \max(0,\vx^{i-1})$, we propagate the lower and upper bound separately, resulting in an output \boxd with $\dot{\vx}^{i} = \tfrac{\obar{\vx}^{i} + \ubar{\vx}^{i}}{2}$ and $\bs{\delta}^i = \tfrac{\obar{\vx}^{i} - \ubar{\vx}^{i}}{2}$ where  $\ubar{\vx}^{i} = \relu(\dot{\vx}^{i-1} - \bs{\delta}^{i-1})$ and $\obar{\vx}^{i} = \relu(\dot{\vx}^{i-1} + \bs{\delta}^{i-1})$.
We proceed this way for all layers obtaining first lower and upper bounds on the network's output $\vy$ and then an upper bound $\obar{\vy}^\Delta$ on the logit difference  $y^\Delta_i := y_i - y_t$. Showing that $\obar{y}^\Delta_i < 0, \; \forall i \neq t$ is sufficient to prove adversarial robustness on the considered input region, with $t$ as the unique winning class.

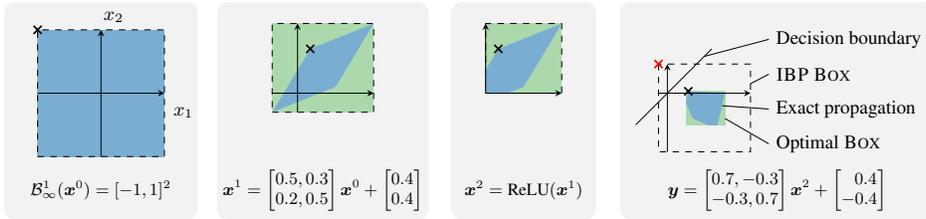
\begin{figure}
\captionsetup{font=footnotesize,skip=4pt}
\vspace{-40mm}
\centering
\resizebox{.9\textwidth}{!}{
\tikzset{
	point/.style={
		thick,
		draw=black,
		cross out,
		inner sep=0pt,
		minimum width=3pt,
		minimum height=3pt,
	},
	point_r/.style={
		thick,
		draw=red,
		cross out,
		inner sep=0pt,
		minimum width=3pt,
		minimum height=3pt,
	},
}

\begin{tikzpicture}[scale=0.70]
\tikzset{>=latex}

	\def \dy{0.2cm}
	\def \ddy{-2.3cm}
	\def \dx{-8cm}
	\def \ddx{5cm}
	\colorlet{c_exact}{my-full-blue!60}
	\colorlet{c_ibp}{c5}
	\colorlet{c_opt}{my-full-green!40}

    \begin{scope}[xshift=\dx, yshift=\dy]
		\node[fill=black!05, rectangle, rounded corners=5pt, minimum width=3.2cm, minimum height=3.6cm] at (0, 0.18*\ddy) {};
		\begin{scope}[yshift=0.0cm, scale=1.5]
			\coordinate (O) at ({0.0000},{0.0000});
			\coordinate (X) at ({1.0000},{0.0000});
			\coordinate (Y) at ({0.0000},{1.0000});
			\coordinate (XO) at ({-1.0000},{0.0000});
			\coordinate (YO) at ({0.0000},{-1.0000});

			\coordinate (input_head) at ({0.0000},{0.0000});
			\coordinate (input_p_0) at ({-1.0000},{-1.0000});
			\coordinate (input_p_1) at ({1.0000},{-1.0000});
			\coordinate (input_p_2) at ({1.0000},{1.0000});
			\coordinate (input_p_3) at ({-1.0000},{1.0000});
	
			\fill [fill=c_exact, opacity=1.0, rounded corners=0mm] (input_p_0) -- (input_p_1) -- (input_p_2) -- (input_p_3)  -- cycle;
			\draw [black, dashed, opacity=1.0, rounded corners=0mm] (input_p_0) -- (input_p_1) -- (input_p_2) -- (input_p_3)  -- cycle;
			\draw[-stealth] (XO) -- (X);
			\draw[-stealth] (YO) -- (Y);
			
			\node at ($(X)+(0.3,-0.3)$) {\footnotesize$x_1$};
			\node at ($(Y)+(0.2,0.2)$) {\footnotesize$x_2$};
			\node[point] at (input_p_3) {};
	
		\end{scope}
		\node[anchor=center,scale=0.9] at (0, \ddy) {\footnotesize $\bc{B}_\infty^{1}(\vx^0) = [-1,1]^2$};
	\end{scope}
		
	\begin{scope}[xshift=\dx + 1.05*\ddx, yshift=\dy]
		\node[fill=black!05, rectangle, rounded corners=5pt, minimum width=3.5cm, minimum height=3.6cm] at (0, 0.18*\ddy) {};

		\begin{scope}[xshift=-0.6cm,yshift=0.0cm, scale=1.5]
			\coordinate (O) at ({0.0000},{0.0000});
			\coordinate (X) at ({1.2000},{0.0000});
			\coordinate (Y) at ({0.0000},{1.1000});
			\coordinate (XO) at ({-0.4000},{0.0000});
			\coordinate (YO) at ({0.0000},{-0.3000});
			
			\coordinate (linear_box_p_0) at ({-0.4000},{-0.3000});
			\coordinate (linear_box_p_1) at ({1.2000},{-0.3000});
			\coordinate (linear_box_p_2) at ({1.2000},{1.1000});
			\coordinate (linear_box_p_3) at ({-0.4000},{1.1000});
			\fill [fill=c_opt, opacity=1.0, rounded corners=0mm] (linear_box_p_0) -- (linear_box_p_1) -- (linear_box_p_2) -- (linear_box_p_3)  -- cycle;
			\draw [black, dashed, opacity=1.0, rounded corners=0mm] (linear_box_p_0) -- (linear_box_p_1) -- (linear_box_p_2) -- (linear_box_p_3)  -- cycle;

			\coordinate (linear_head) at ({0.4000},{0.4000});
			\coordinate (linear_p_0) at ({0.2000},{0.7000});
			\coordinate (linear_p_1) at ({-0.4000},{-0.3000});
			\coordinate (linear_p_2) at ({0.6000},{0.1000});
			\coordinate (linear_p_3) at ({1.2000},{1.1000});
			
			\fill [fill=c_exact, opacity=1.0, rounded corners=0mm] (linear_p_0) -- (linear_p_1) -- (linear_p_2) -- (linear_p_3)  -- cycle;
			
			\node[point] at (linear_p_0) {};
			
			\draw[-stealth] (XO) -- (X);
			\draw[-stealth] (YO) -- (Y);
		\end{scope}
		\node[anchor=center,scale=0.9] at (0, \ddy) {\footnotesize $\vx^1 = \begin{bmatrix}
			0.5, 0.3 \\
			0.2, 0.5 \\
			\end{bmatrix}\vx^0 + \begin{bmatrix}
			0.4\\0.4
			\end{bmatrix}$};
		
	\end{scope}
	
	\begin{scope}[xshift=\dx + 2.0 * \ddx, yshift=\dy]
		\node[fill=black!05, rectangle, rounded corners=5pt, minimum width=2.4cm, minimum height=3.6cm] at (0, 0.18*\ddy) {};
		\begin{scope}[xshift = -0.9cm, yshift=0.0cm, scale=1.5]
		\coordinate (O) at ({0.0000},{0.0000});
		\coordinate (X) at ({1.2000},{0.0000});
		\coordinate (Y) at ({0.0000},{1.1000});
		\coordinate (XO) at ({-0.0000},{0.0000});
		\coordinate (YO) at ({0.0000},{-0.0000});
		\coordinate (XX) at ({5.0000},{0.0000});
		\coordinate (YY) at ({0.0000},{5.0000});
		
		\coordinate (linear_box_p_0) at ({-0.000},{-0.000});
		\coordinate (linear_box_p_1) at ({1.2000},{-0.000});
		\coordinate (linear_box_p_2) at ({1.2000},{1.1000});
		\coordinate (linear_box_p_3) at ({-0.000},{1.1000});
		\fill [fill=c_opt, opacity=1.0, rounded corners=0mm] (linear_box_p_0) -- (linear_box_p_1) -- (linear_box_p_2) -- (linear_box_p_3)  -- cycle;
		\draw [black, dashed, opacity=1.0, rounded corners=0mm] (linear_box_p_0) -- (linear_box_p_1) -- (linear_box_p_2) -- (linear_box_p_3)  -- cycle;

		\coordinate (linear_p_0) at ({0.2000},{0.7000});
		\coordinate (linear_p_1) at ({-0.4000},{-0.3000});
		\coordinate (linear_p_2) at ({0.6000},{0.1000});
		\coordinate (linear_p_3) at ({1.2000},{1.1000});

  		\path[name path=x] (O) -- (XX);
  		\path[name path=y] (O) -- (YY);
  		\path[name path=top] (linear_p_1) -- (linear_p_0);
  		\path[name path=bottom] (linear_p_1) -- (linear_p_2);
		\path[name intersections={of=y and top}];
		\coordinate (linear_p_4) at (intersection-1);
		\path[name intersections={of=x and bottom}];
		\coordinate (linear_p_5) at (intersection-1);

		\fill [fill=c_exact, opacity=1.0, rounded corners=0mm] (linear_p_0) -- (linear_p_4) --(O) -- (linear_p_5) -- (linear_p_2) -- (linear_p_3)  -- cycle;
		
		\node[point] at (linear_p_0) {};
				
		\draw[-stealth] (XO) -- (X);
		\draw[-stealth] (YO) -- (Y);
		\end{scope}
		\node[anchor=center,scale=0.9] at (0, \ddy) {\footnotesize$\vx^2 = \text{ReLU}(\vx^1)$};
	
	\end{scope}

	\begin{scope}[xshift=\dx + 3.2*\ddx, yshift=\dy]
		\node[fill=black!05, rectangle, rounded corners=5pt, minimum width=5.2cm, minimum height=3.6cm] at (0, 0.18*\ddy) {};

		\begin{scope}[xshift=-2.6cm, yshift=0.0cm, scale=1.5]
			\coordinate (O) at ({0.0000},{0.0000});
			\coordinate (DA) at ({-0.500},{-0.500});
			\coordinate (DB) at ({0.700},{0.700});
			\coordinate (X) at ({1.3107},{0.0000});
			\coordinate (Y) at ({0.0000},{0.4600});
			\coordinate (XO) at ({-0.1400},{0.0000});
			\coordinate (YO) at ({0.0000},{-0.9250});					
			\coordinate (XX) at ({5.0000},{0.0000});
			\coordinate (YY) at ({0.0000},{5.0000});
			
			\coordinate (box_2_p_0) at ({-0.1400},{0.4600});
			\coordinate (box_2_p_1) at ({-0.1400},{-0.9250});
			\coordinate (box_2_p_2) at ({1.3107},{-0.9250});
			\coordinate (box_2_p_3) at ({1.3107},{0.4600});
			\draw [black, dashed, opacity=1.0, rounded corners=0mm] (box_2_p_0) -- (box_2_p_1) -- (box_2_p_2) -- (box_2_p_3)  -- cycle;

			\coordinate (linear_2_p_0) at ({0.7900},{-0.5100});
			\coordinate (linear_2_p_1) at ({0.9100},{0.0100});
			\coordinate (linear_2_p_2) at ({0.3300},{0.0300});
			\coordinate (linear_2_p_3) at ({0.2900},{-0.1433});
			\coordinate (linear_2_p_4) at ({0.4000},{-0.4000});
			\coordinate (linear_2_p_5) at ({0.6450},{-0.5050});

			\coordinate (box_3_p_0) at ({0.2900},{0.0300});
			\coordinate (box_3_p_1) at ({0.2900},{-0.5100});
			\coordinate (box_3_p_2) at ({0.9100},{-0.5100});
			\coordinate (box_3_p_3) at ({0.9100},{0.0300});

			\fill [fill=c_opt, opacity=1.0, rounded corners=0mm] (box_3_p_0) -- (box_3_p_1) -- (box_3_p_2) -- (box_3_p_3)  -- cycle;
			\fill [fill=c_exact, opacity=1.0, rounded corners=0mm] (linear_2_p_0) -- (linear_2_p_1) -- (linear_2_p_2) -- (linear_2_p_3) -- (linear_2_p_4) -- (linear_2_p_5)  -- cycle;

			\draw[black] ($(DA)!0.90!(DB)$) -- (1.6, 0.85) node[right, font=\footnotesize, text=black] {Decision boundary};
			\draw[black] (1.3107, 0.15) -- (1.6, 0.3) node[right, font=\footnotesize, text=black] {\ibp \boxd};
			\draw[black] ($(linear_2_p_0)!0.60!(linear_2_p_1)$) -- (1.6, -0.25) node[right, font=\footnotesize, text=black] {Exact propagation};
			\draw[black] (0.91, -0.4) -- (1.6, -0.8) node[right, font=\footnotesize, text=black] {Optimal \boxd};
			
			\node[point] at (linear_2_p_2) {};
			\node[point_r] at (box_2_p_0) {};

			\draw[-stealth] (XO) -- (X);
			\draw[-stealth] (YO) -- (Y);
			\draw[-,solid] (DA) -- (DB);
		\end{scope}

\node[anchor=center,scale=0.9] at (0, \ddy) {\footnotesize $\vy = \begin{bmatrix}
			0.7, -0.3 \\
			-0.3, 0.7 \\
			\end{bmatrix}\vx^2 + \begin{bmatrix}
			\;\;\, 0.4\\-0.4
			\end{bmatrix}$};

\end{scope}

	\coordinate (z) at ({0.1231},{0.0846});
\end{tikzpicture}
}
\caption{Comparison of exact (\markerexact), optimal box (\markeropt), and \ibp (\markeribp) propagation through a one layer network. We show the concrete points maximizing the logit difference $y_2-y_1$ as a black $\times$ and the corresponding relaxation as a red \textcolor{red}{$\times$}.}%
\label{fig:overview}
\end{figure}

\Cref{fig:overview} illustrates a two-layer network with exact propagation \markerexact in blue, optimal boxes \markeropt in green, and dashed \ibp boxes \markeribp. After the first linear and ReLU layer (third column), even the smallest box contains unreachable points. Layer-wise propagation accumulates these errors, yielding looser bounds than optimal propagation after the final linear layer (rightmost column).
Certifying class $1$ throughout $[-1,1]^2$ requires an upper bound below zero on $y_2-y_1$. Its concrete maximum is $-0.3$ (black $\times$), proving robustness, but the \ibp upper bound is $0.6$ (red {\color{red}$\times$}) and cannot certify it. The optimal box bounds suffice to prove robustness.

\paragraph{Training for Robustness} is required to obtain (certifiably) robust neural networks.
For a data distribution $(\vx, t) \sim \bc{D}$, standard training optimizes the network parametrization $\bs{\theta}$ to minimize the expected cross-entropy loss:
\begin{equation}\label{eq:std_train}
	\theta_\text{std} = \argmin_\theta \E_\bc{D} [\bc{L}_\text{CE}(\vf_{\bs{\theta}}(\vx),t)], \quad \text{with} \quad \bc{L}_\text{CE}(\vy, t) = \ln\big(1 + \sum_{i \neq t} \exp(y_i-y_t)\big).
\end{equation}
To train for robustness, we, instead, aim to minimize the expected \emph{worst-case loss} for a given robustness specification, leading to a min-max optimization problem:
\begin{equation}
\label{eq:rob_opt}
	\theta_\text{rob} = \argmin_\theta \mathbb{E}_{\bc{D}} \left[ \max_{\vx' \in \bc{B}^{\epsilon}(\vx) }\bc{L}_\text{CE}(\vf_{\bs{\theta}}(\vx'),t) \right].
\end{equation}
As computing the worst-case loss by solving the inner maximization problem is generally intractable, it is commonly under- or over-approximated, yielding adversarial and certified training, respectively. 

\paragraph{Adversarial Training} optimizes a lower bound on the inner optimization objective in \cref{eq:rob_opt}. It first computes concrete examples $\vx'\in \bc{B}^{\epsilon}(\vx)$ that approximately maximize the loss term $\bc{L}_\text{CE}$ and then optimizes the network parameters $\bs{\theta}$ for these examples.
While networks trained this way typically exhibit good empirical robustness, they remain hard to formally certify and are sometimes vulnerable to stronger attacks \citep{TramerCBM20,Croce020a}.

\paragraph{Certified Training} typically optimizes an upper bound on the inner maximization objective in \cref{eq:rob_opt}. The resulting robust cross-entropy loss $\bc{L}_\text{CE,rob}(\bc{B}^{\epsilon}(\vx),t) = \bc{L}_\text{CE}(\overline{\vy}^\Delta,t)$ is obtained by first computing an upper bound $\overline{\vy}^\Delta$ on the logit differences $\vy^\Delta := \vy - y_t$ with a bound propagation method as described above and then plugging it into the standard cross-entropy loss.

Surprisingly, the imprecise \ibp bounds \citep{MirmanGV18,GowalIBP2018,ShiWZYH21} consistently yield better performance than methods based on tighter approximations \citep{WongSMK18,ZhangCXGSLBH20,balunovic2020Adversarial}. \citet{jovanovic2022paradox} trace this back to the optimization problems induced by the more precise methods becoming intractable to solve. 

However, the heavy regularization that makes \ibp trained networks amenable to certification also severely reduces their standard accuracy. To alleviate the resulting robustness-accuracy trade-off, \emph{all} current state-of-the-art certified training methods combine \ibp and adversarial training by using \ibp bounds only for regularization (\ibpr \citep{PalmaIBPR22}), by only propagating small, adversarially selected regions (\sabr \citep{MuellerEFV22}), using \ibp bounds only for the first layers and \pgd bounds for the remainder of the network (\taps \citep{MaoMFV2023}), or combining losses over adversarial samples and \ibp bounds (\ccibp, \mtlibp \citep{PalmaBDKSL23}).

In light of this surprising dominance of \ibp-based training methods, understanding the regularization \ibp induces and its effect on tightness promises to be a key step towards developing novel and more effective certified training methods.

\section{Understanding \ibp Training}

We compare layer-wise (\ibp) and optimal box propagation, illustrated in \cref{fig:overview}. Optimal bounds on logit differences enable exact robustness certification (\cref{lem:opt_box_bound}).
After defining both forms of propagation, we derive necessary and sufficient conditions for identical bounds. Their strong regularization motivates propagation tightness $\tau$, a relaxed measure of precision computable efficiently globally for deep linear networks (DLNs) and locally for ReLU networks. We study its architectural dependence at initialization and its increase under \ibp training. Finally, we show that linear dimensionality reduction is inherently imprecise for both propagation methods, making sufficient width essential for tight bounds. All proofs are deferred to \cref{app:proofs}.

\paragraph{Setting} Deep linear networks (DLNs), $\vf(x) = \Pi_{i=1}^L \mW^{(i)}\vx$, provide a tractable model for neural network analysis \citep{SaxeMG13, JiT19, WuWM19}. They retain ReLU networks' layer-wise structure and joint non-convexity in the layer weights \citep{Ribeiro0G16}. Although globally linear, they give exact local box radii for sufficiently small perturbations when the ReLU activation pattern is locally constant.
For the theoretical results below, we use closed full boxes $\B^{\bm{\epsilon}}(\vx)=[\vx-\bm{\epsilon},\vx+\bm{\epsilon}]$ with $\bm{\epsilon}\ge\bm{0}$. These over-approximate the domain-clipped perturbation sets in \cref{sec:background}. After proving key results on DLNs, we will show how they transfer to ReLU networks.

\subsection{Layer-wise and Optimal Box Propagation}
We define the optimal hyper-box approximation $\ibpb(\vf, \B^{\bm{\epsilon}}(\vx))$ as the smallest hyper-box $[\underline{\vz}, \overline{\vz}]$ such that it contains the image $\vf(\vx^\prime)$ of all points $\vx^\prime$ in $\B^{\bm{\epsilon}}(\vx)$, \ie, $\vf(\vx^\prime) \in [\underline{\vz}, \overline{\vz}],\forall \vx^\prime \in \B^{\bm{\epsilon}}(\vx)$. Similarly, we define the layer-wise box approximation as the result of sequentially applying the optimal approximation to every layer individually: $\ibpl(\vf,\B^{\bm{\epsilon}}(\vx)) \coloneqq \ibpb(\mW_L,\ibpb(\cdots, \ibpb(\mW^{(1)},\B^{\bm{\epsilon}}(\vx))))$. We write their upper- and lower-bounds as $[\lbopt, \ubopt]$ and $[\lbbox, \ubbox]$, respectively.
We note that optimal box bounds on the logit differences $\vy^\Delta \coloneqq \vy - y_t$ (instead of on the logits $\vy$ as shown in \cref{fig:overview}) are sufficient for exact robustness verification:
\begin{restatable}{lem}{exactcert}
        \label[lemma]{lem:opt_box_bound}
        Any $\bc{C}^0$ continuous classifier $\vf$, computing the logit difference $y^\Delta_i := y_i - y_t, \forall i \neq t$, is robustly correct with unique winning class $t$ on $\B^{\bm{\epsilon}}(\vx)$ if and only if $\ibpb(\vf, \B^{\bm{\epsilon}}(\vx)) \subseteq \R_{< 0}^{c-1}$, \ie $\obar{y}^{\Delta^*}_i < 0 , \forall i \neq t$.
\end{restatable}
For DLNs, we can efficiently compute both optimal $\ibpb$ and layerwise $\ibpl$ box bounds as follows:
\vspace{-3mm}
\begin{restatable}[Box Propagation]{thm}{boxprop}
    \label{lem:box_size_L}
    For an $L$-layer DLN $\vf = \Pi_{k=1}^L \mW^{(k)}$, we obtain the box centres $\dot{\vz}^* = \dot{\vz}^\dagger = \vf(\vx)$ and the radii
    \begin{equation}
        \frac{\ubopt - \lbopt}{2} =  \left| \Pi_{k=1}^L \mW^{(k)} \right| \bm{\epsilon}, \quad \text{and} \quad \frac{\ubbox - \lbbox}{2} = \left(\Pi_{k=1}^L \left|\mW^{(k)} \right|\right) \bm{\epsilon}.
    \end{equation} 
\end{restatable}
Comparing the radius computation of the optimal and layer-wise approximations, we observe that the main difference lies in where the element-wise absolute value $|\cdot|$ of the weight matrix is taken. For the optimal box, we first multiply all weight matrices before taking the absolute value $|\Pi_{k=1}^L \mW^{(k)}|$, thus allowing for cancellations of terms of opposite signs. For the layer-wise approximation, in contrast, we first take the absolute value of each weight matrix before multiplying them together $\Pi_{k=1}^L |\mW^{(k)}|$, thereby losing all relational information between variables. Let us now investigate under which conditions layer-wise and optimal bounds become identical.

\subsection{Propagation Invariance and \ibp Bound Tightness}
\label{sec:pi_conditions}
\paragraph{Propagation Invariance} We call a network (globally) \emph{propagation invariant} (PI) if the layer-wise and optimal box over-approximations are identical for every input box.
Clearly, non-negative weight matrices lead to PI networks \citep{LinIWSL22}, as the absolute value in \cref{lem:box_size_L} loses its effect.
However, non-negative weights significantly reduce network expressiveness and performance \citep{ChorowskiZ14}, raising the question of whether they are a necessary condition.
We show that they are not, by deriving a sufficient \emph{and} necessary condition for a two-layer DLN:
\begin{restatable}[Propagation Invariance]{lem}{invariant}
    \label[lemma]{lem:invariant}
    A two-layer DLN $\vf = \mW^{(2)} \mW^{(1)}$ is propagation invariant if and only if for every fixed $(i, j)$, we have $\left| \sum_k W^{(2)}_{i,k} \cdot  W^{(1)}_{k,j} \right| = \sum_k |W^{(2)}_{i,k} \cdot  W^{(1)}_{k,j}|$, \ie, either $W^{(2)}_{i, k} \cdot W^{(1)}_{k, j} \ge 0$ for all $k$ or $W^{(2)}_{i, k} \cdot W^{(1)}_{k, j} \le 0$ for all $k$.
\end{restatable}
\paragraph{Conditions for Propagation Invariance} To see how strict the condition described by \cref{lem:invariant} is, we observe that propagation invariance requires the sign of the last element in any two-by-two block in $\mW^{(2)} \mW^{(1)}$ to be determined by the signs of the other three elements:
\begin{restatable}[Non-Propagation Invariance]{thm}{noninvariant}
    \label{thm:invariant-reg}
     Assume $\exists i, i^\prime, j, j^\prime$, such that $W^{(1)}_{k,j}$, $W^{(1)}_{k,j^\prime}$, $W^{(2)}_{i,k}$ and $W^{(2)}_{i^\prime,k}$ are nonzero for every hidden index $k$. If $(\mW^{(2)} \mW^{(1)})_{i,j} \cdot (\mW^{(2)} \mW^{(1)})_{i, j^\prime} \cdot (\mW^{(2)} \mW^{(1)})_{i^\prime, j} \cdot (\mW^{(2)} \mW^{(1)})_{i^\prime, j^\prime} < 0$, then $\vf = \mW^{(2)} \mW^{(1)}$ is not propagation invariant.
\end{restatable}
To obtain a propagation invariant network with weights $\mW^{(2)} \mW^{(1)} \in \mathcal{R}^{d\times d}$, we can thus only choose $2d-1$ (e.g., one row and one column) of the $d^2$ signs freely  (see \cref{cor:linear_param} in \cref*{app:theory}).

The statements of \cref{lem:invariant,thm:invariant-reg} naturally extend to DLNs with more than two layers $L>2$. However, the conditions within \cref{thm:invariant-reg} become increasingly complex and strict as more and more terms need to yield the same sign. Thus, we focus our analysis on $L=2$ for clarity.

\paragraph{\ibp Bound Tightness}
To analyze the tightness of \ibp bounds for networks that do not satisfy the strict conditions for propagation invariance, we relax it to introduce \emph{propagation tightness} as the ratio between the optimal and layer-wise box radii, simply referred to as \emph{tightness} in this paper. 
\begin{mydef} \label{def:tightness}
    Given a DLN $\vf$, we define the global propagation tightness $\boldsymbol{\tau}$ as the ratio between optimal $\ibpb(\vf, \B^{\bm{\epsilon}}(\vx))$ and layer-wise $\ibpl(\vf, \B^{\bm{\epsilon}}(\vx))$ approximation radius, \ie, $\boldsymbol{\tau} = \frac{\ubopt - \lbopt}{\ubbox - \lbbox}$, with coordinatewise division wherever the layer-wise radius is positive.
\end{mydef}
\begin{wrapfigure}[13]{r}{0.29\linewidth}
    \captionsetup{font=footnotesize,skip=4pt}
    \vspace{-5mm}
    \centering
    \includegraphics[width=0.95\linewidth]{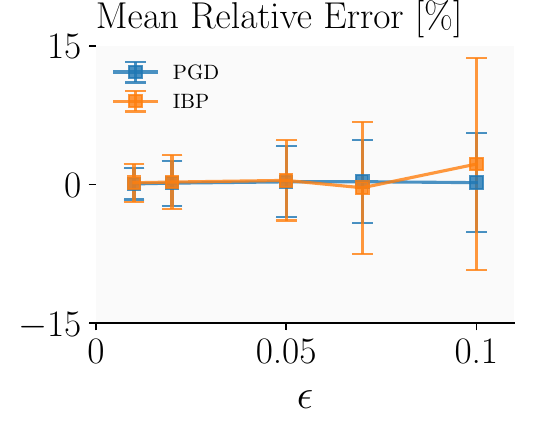}
    \caption{Mean relative error between local tightness (\cref{def:relu_tightness}) and true tightness computed with \milp for a \cnnt trained with \pgd or \ibp at $\epsilon=0.05$ on \mnist.}
    \label{fig:tightness_approximation_error}
\end{wrapfigure}
Intuitively, tightness measures how much smaller the exact dimension-wise $\ibpb$ bounds are, compared to the layer-wise approximation $\ibpl$, thus quantifying the gap between \ibp certified and true adversarial robustness. When tightness equals $1$, the network is propagation invariant and can be certified exactly with \ibp; when tightness is close to $0$, \ibp bounds become arbitrarily imprecise. We highlight that this is orthogonal to the box diameter $\Delta = \ubbox - \lbbox$, considered by \citet{ShiWZYH21}.

\paragraph{ReLU Networks}  
The nonlinearity of ReLU networks leads to locally varying tightness and makes the computation of optimal box bounds intractable. At inputs with nonzero ReLU preactivations, activation patterns remain constant for sufficiently small perturbation magnitudes $\epsilon$, making the network locally affine. We thus introduce a local version of tightness around concrete inputs.
\begin{mydef} \label{def:relu_tightness}
    Consider an $L$-layer ReLU network with a linear output layer and a locally constant activation pattern at $\vx$.
    Set $\vx^{(0)}=\vx$, $\va^{(k)}=\mW^{(k)}\vx^{(k-1)}+\vb^{(k)}$, and $\vx^{(k)}=\relu(\va^{(k)})$ for $k<L$.
    Define $\vd^{(k)}(\vx)=\mathds{1}_{\va^{(k)}>0}$ for $k<L$ and $\vd^{(L)}=\bm{1}$ for the linear output layer.
    For isotropic input radius $\epsilon\bm{1}$, its \emph{local tightness} is
    \begin{equation*}
        \bm{\tau} = \frac{\frac{d}{d \epsilon}(\ubopt - \lbopt)\big|_{\epsilon=0^+}}{\frac{d}{d \epsilon}(\ubbox - \lbbox)\big|_{\epsilon=0^+}}=
        \frac{\left|\Pi_{k=1}^L \diag(\vd^{(k)}) \mW^{(k)} \right|\bm{1}}{(\Pi_{k=1}^L \diag(\vd^{(k)}) \left|\mW^{(k)} \right|) \bm{1}}.
    \end{equation*}
    Derivatives are taken from the right, and division is coordinatewise wherever the denominator is positive.
\end{mydef}

In \cref{def:relu_tightness}, we calculate tightness as the ratio of box size growth rates, evaluated for an infinitesimal input box size $\epsilon$. Under the locally constant activation assumption, sufficiently small input boxes contain no unstable neurons, making the radius analysis exact. Larger boxes can cross activation boundaries, making local tightness an approximation of tightness at that $\epsilon$.
However, for the networks and perturbation magnitudes typically considered in the literature, only a very small fraction ($\approx 1\%$) of neurons are unstable \citep{MuellerEFV22}. To assess the estimation quality of our local tightness, we show its mean relative error compared to the exact tightness computed with \milp for a small \cnnt in \cref{fig:tightness_approximation_error} for \mnist and in \cref{fig:tightness_approximation_error_cifar} for \cifar. We find that for perturbations smaller than those used during training ($\epsilon \leq 0.05$) relative errors are extremely small ($<0.5\%$), and only increase slowly after, reaching $2.2\%$ at $\epsilon=0.1$.

\subsection{Tightness at Initialization} \label{sec:initialization}
We first investigate the ratio of expected radii $\tau = \tfrac{\E_{\mathcal{D}_{\boldsymbol{\theta}}}(\ubopt - \lbopt)}{\E_{\mathcal{D}_{\boldsymbol{\theta}}}(\ubbox - \lbbox)}$ (independent of the dimension due to symmetry) at initialization, \ie, w.r.t. a weight distribution $\mathcal{D}_{\boldsymbol{\theta}}$. This differs in general from the expectation of a samplewise tightness ratio.
Let us consider a two-layer DLN at initialization, \ie, with i.i.d. weights following a zero-mean Gaussian distribution $\mathcal{N}(0, \sigma^2)$ with an arbitrary but fixed positive variance $\sigma^2$ \citep{GlorotB10,HeZRS15}.
\begin{restatable}[Initialization Tightness w.r.t.~Width]{lem}{initwidth}
    \label{thm:tightnessinit}
    Given a 2-layer DLN with weight matrices $\mW^{(1)} \in \mathcal{R}^{d_1\times d_0}$, $\mW^{(2)} \in \mathcal{R}^{d_2\times d_1}$ with i.i.d. entries from $\mathcal{N}(0, \sigma_1^2)$ and $\mathcal{N}(0, \sigma_2^2)$ independently across layers, with $\sigma_1,\sigma_2>0$ (together denoted as $\boldsymbol{\theta}$), we obtain the expected tightness $\tau (d_1) = \frac{\E_{\boldsymbol{\theta}}(\ubopt - \lbopt)}{\E_{\boldsymbol{\theta}}(\ubbox - \lbbox)} = \frac{\sqrt{\pi} \, \Gamma(\frac{1}{2}(d_1+1))}{d_1 \Gamma(\frac{1}{2}d_1)} \in {\Theta}(\frac{1}{\sqrt{d_1}})$.
\end{restatable}
Tightness at initialization, thus, decreases quickly with internal width ($\Theta(\frac{1}{\sqrt{d_1}})$), e.g., by a factor of $\tau(500) \approx 0.056$ for the penultimate layer of the popular \cnns \citep{GowalIBP2018,ZhangCXGSLBH20}. %
It, further, follows directly that tightness will decrease exponentially w.r.t. network depth.
\begin{restatable}[Initialization Tightness w.r.t. Depth]{cor}{initdepth}
    \label[corollary]{cor:exp_growth}
	Under independent centered Gaussian initialization with positive layer variances, the ratio of expected radii of an $L$-layer DLN $\vf$, $L\ge2$, with minimum internal dimension $d_\text{min}$ satisfies $\tau \leq \tau(d_\text{min})^{\lfloor L/2 \rfloor}$.
\end{restatable}
This result is independent of the variance $\sigma_1^2,\sigma_2^2$. Thus, tightness at initialization can not be increased by scaling $\sigma^2$, as proposed by \citet{ShiWZYH21} to achieve constant box radius over network depth. 

\paragraph{ReLU Networks}
We extend \cref{thm:tightnessinit} to two-layer ReLU networks (\cref{cor:2_layer_relu_init} in \cref{app:theory}), obtain $\tau^\prime(d_1)\sim\sqrt{2}\tau(d_1)$ as $d_1\to\infty$, and empirically validate it in \cref{sec:network_config}.

\subsection{IBP Training Increases Tightness}
We now show theoretically that \ibp training increases this tightness.
To this end, we again consider a DLN with layer-wise propagation matrix $\ibpW = \Pi_{i=1}^L |\mW^{(i)}|$ and optimal propagation matrix $\optW = |\Pi_{i=1}^L \mW^{(i)}|$, yielding the expected risk for \ibp training as $R(\bm{\epsilon}) = \E_{\vx,y} \L(\ibpl(\vf, \B^{\bm{\epsilon}}(\vx)), y)$.%

\begin{restatable}[IBP Training Increases Tightness]{thm}{ibptraining}
    \label{thm:radius_tightness}
    Consider a DLN with cross-entropy risk $R(\bm{\epsilon})$ on the IBP output box and $\|\vx\|_\infty\le1$ almost surely. Assume differentiable propagation matrices and locally differentiable homogeneous tightness $\optW=\tau\ibpW$, with $\tau>0$ and $\ibpW_{ij}>0$.
    If $\frac{\|\nabla_\theta\optW_{ij}\|_2}{\optW_{ij}}\le c\frac{\|\nabla_\theta\ibpW_{ij}\|_2}{\ibpW_{ij}}$ for all $i,j$ and some $c\in[0,1)$ satisfying $2Mc(1+c)\le1-c$, where $M=\max_i\sum_j\optW_{ij}$, then $\langle\nabla_\theta(R(\bm{\epsilon})-R(\bm{0})),\nabla_\theta\tau\rangle\le0$ for all $\bm{\epsilon}\ge\bm{0}$.
    In particular, $c\le(1+4M)^{-1}$ suffices for the condition on $M$.
\end{restatable}

\subsection{Network Width and Tightness after Training}

Many high-dimensional computer vision datasets were shown to have low intrinsic data dimensionality \citep{PopeZAGG21}. Thus, we study the reconstruction loss of a linear embedding into a low-dimensional subspace as a proxy for performance and find that tightness decreases with the width $w$ of a bottleneck layer as long as it is smaller than the data-dimensionality $d$, \ie, $w \ll d$. Further, while reconstruction becomes lossless for points as soon as the width $w$ reaches the intrinsic dimension $k$ of the data, even optimal box propagation requires a width of at least the original data dimension $d$ to achieve loss-less reconstruction.
For a $k$-dimensional data distribution, linearly embedded into a $d$ dimensional space with $d \gg k$, the data matrix $X$ has a low-rank eigendecomposition $\Var(X) = U \Lambda U^\top$ with $k$ non-zero eigenvalues. %
The optimal reconstruction $\hat{X} = U_k U^\top_k X$ is exact by choosing $U_k$ as the $k$ columns of $U$ with non-zero eigenvalues. Yet, box propagation is imprecise:

\begin{restatable}[Box Reconstruction Error]{thm}{embedding}
    \label{thm:ibp_reconstruction}
    Consider the linear embedding and reconstruction $\hat{\vx} = U_k U_k^\top \vx$ of a $d$ dimensional data distribution $\vx \sim \bc{X}$ into a $k$ dimensional space with $1\le k\le d$ and eigenmatrices $U$ drawn uniformly at random from the orthogonal group. For scalar radius $\epsilon>0$, propagating the input box $\B^{\epsilon\bm{1}}(\vx)$ layer-wise and optimally yields $\B^{\bm{\delta}^\dagger}\!(\hat{\vx})$, and $\B^{\bm{\delta}^*}\!(\hat{\vx})$, respectively.
    Then (i) $\E(\delta_i^\dagger / \epsilon) = ck \in \Theta(k)$, where $c=\frac1d+\frac{2(d-1)}{\pi d}\to\frac2\pi$ as $d\to\infty$.
    Further, (ii) for fixed $k$ and $d\to\infty$, $\E(\delta^*_i / \epsilon)\to\frac{2}{\sqrt{\pi}}\frac{\Gamma((k+1)/2)}{\Gamma(k/2)}$. This limiting expression is $\Theta(\sqrt{k})$ as $k\to\infty$.
\end{restatable}
Intuitively, \cref{thm:ibp_reconstruction} implies that, while input points can be embedded into and reconstructed from a $k$ dimensional space losslessly, in the iterated limit of large $d$ followed by large $k$, box propagation yields radius growth of $\Theta(\sqrt{k})$ for optimal and $\Theta(k)$ for layer-wise propagation. However, with $k=d$, we can choose $U_k$ to be an identity matrix, thus obtaining lossless "reconstruction", highlighting the importance of network width.

\section{Empirical Evaluation Analysis}
\begin{figure}[!htbp]
\captionsetup{font=footnotesize,skip=4pt}
    \begin{minipage}[c]{.64\linewidth}
        \centering
            \begin{subfigure}{.48\linewidth}
            \centering
            \includegraphics[width=\linewidth]{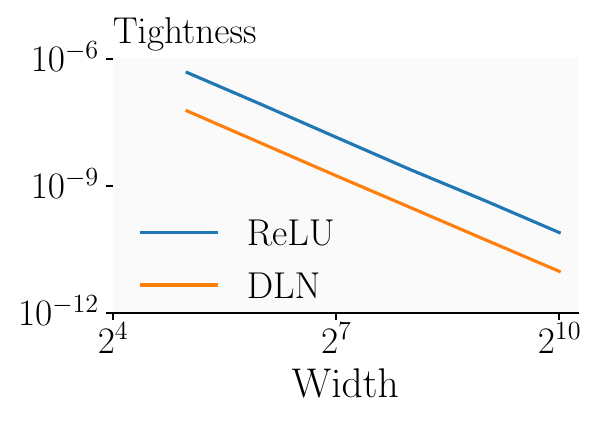}
        \end{subfigure}
        \hfil
        \begin{subfigure}{.48\linewidth}
            \centering
            \includegraphics[width=\linewidth]{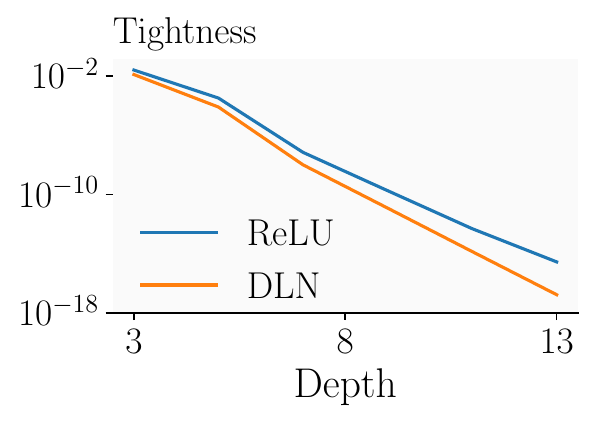}
        \end{subfigure}

        \caption{Dependence of tightness at initialization on width (left) and depth (right) for  a \cnns and \cifar.}
        \label{fig:init_scaling}
            
    \end{minipage}
    \hfil
    \begin{minipage}[c]{.31\linewidth}
        \centering
            \includegraphics[width=0.96\linewidth]{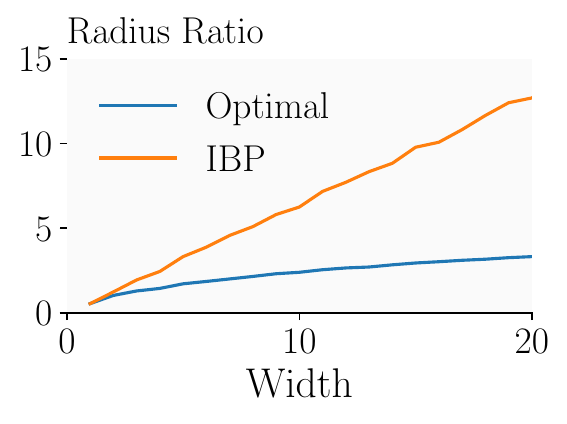}
        \caption{Box reconstruction error over bottleneck width $w$.} 
        \label{fig:linear_reconstruction}
    \end{minipage}
\end{figure}

We use local tightness (\cref{def:relu_tightness}) to study \ibp-based certified training and test our theoretical predictions on ReLU networks.
For certification, we use \mnbab \citep{FerrariMJV22}, a state-of-the-art verifier, and defer further details to \cref{app:exp_details}.

\subsection{Network Architecture and Tightness} \label{sec:network_config}
\begin{wrapfigure}[11]{r}{0.56\textwidth}
\captionsetup{font=footnotesize,skip=4pt}
    \vspace{-5.0mm}
    \centering
    \begin{minipage}[t]{0.45\linewidth}
        \begin{subfigure}{1.0\linewidth}
            \includegraphics[width=.88\linewidth]{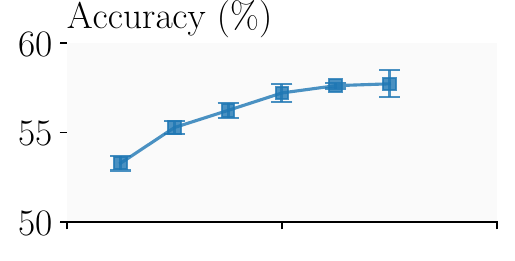}
            \vspace{-1.8mm}
        \end{subfigure}
        \centering
        \begin{subfigure}{1.0\linewidth}
            \centering
            \includegraphics[width=\linewidth]{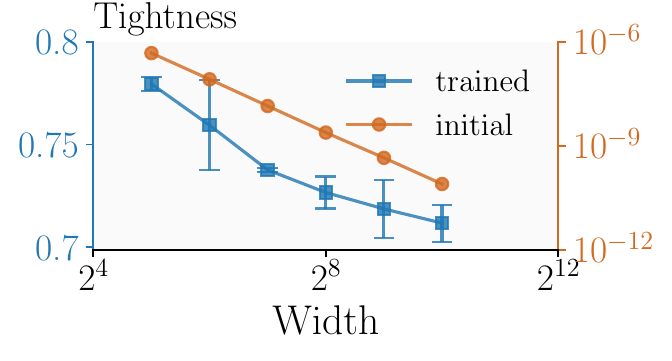}
            \vspace{-5mm}
        \end{subfigure}        
    \end{minipage}
    \hfil
    \begin{minipage}[t]{0.45\linewidth}
        \begin{subfigure}{1.0\linewidth}
            \includegraphics[width=.88\linewidth]{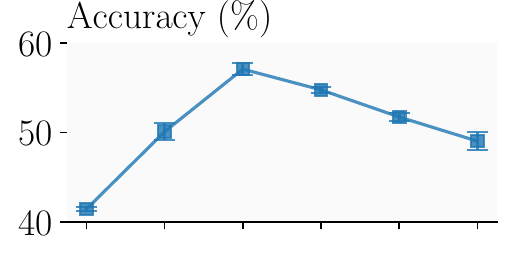}
            \vspace{-1.8mm}
        \end{subfigure}
        \begin{subfigure}{1.0\linewidth}
            \centering
            \includegraphics[width=\linewidth]{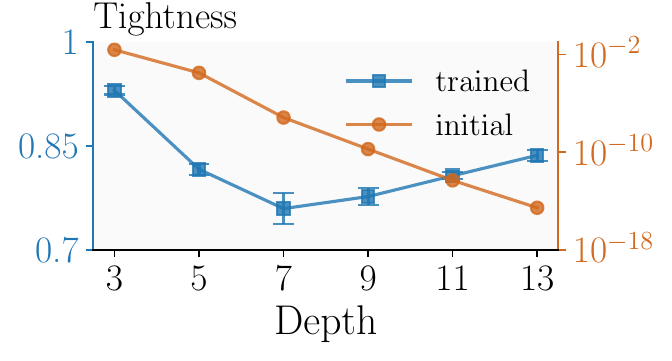}
        \end{subfigure}
    \end{minipage}
    \caption{Effect of network depth (right) and width (left) on tightness and \emph{training set} \ibp -certified accuracy.}
    \label{fig:IBP_architecture_effect}
\end{wrapfigure}

We first test our predictions for width and depth at initialization and after training.
In \cref{fig:init_scaling}, we visualize tightness at initialization, depending on network width and depth for DLNs and ReLU networks. As predicted by \cref{thm:tightnessinit,cor:exp_growth}, tightness decreases polynomially with width (see \cref{fig:init_scaling} left) and exponentially with depth (see \cref{fig:init_scaling} right), both for DLNs and ReLU networks. 
We confirm our results on the inherent hardness of linear reconstruction in \cref{fig:linear_reconstruction}, where we plot the ratio of recovered and original box radii, given a bottleneck layer of width $w$ and synthetic data with intrinsic dimensionality $k=w$. As predicted by \cref{thm:ibp_reconstruction}, \ibp propagation yields linear and \ibpb sublinear growth.%

\begin{wraptable}[30]{R}{.46\textwidth}
    \centering
    \captionsetup{font=footnotesize,skip=4pt}
    \caption{Certified and standard accuracy depending on network width.} \label{tb:sabr_wide}
    \footnotesize
    \setlength{\tabcolsep}{2pt}
    \begin{adjustbox}{max width=\linewidth}
        \begin{tabular}{@{}cclcrr@{}} \toprule
            Dataset                                      & $\epsilon$                                     & Method                              & Width       & Accuracy       & Certified      \\ \midrule

            \multirow{9}{*}{\mnist}                      & \multirow{4.5}{*}{$0.1$}           & \multirow{2}{*}{\ibp}   & $1\times$   & 98.83          & 98.10          \\
                                                         &                                                &                                     & $4\times$   & 98.86          & 98.23          \\
            \cmidrule(lr){3-3}
                                                         &                                                & \multirow{2}{*}{\sabr}  & $1\times$   & 98.99          & 98.20          \\
                                                         &                                                &                                     & $4\times$   & \textbf{98.99} & \textbf{98.32} \\
            \cmidrule(lr){2-3}

                                                         & \multirow{4.5}{*}{$0.3$}                       & \multirow{2}{*}{\ibp}               & $1\times$   & 97.44          & 93.26          \\
                                                         &                                                &                                     & $4\times$   & 97.66          & 93.35          \\
            \cmidrule(lr){3-3}
                                                         &                                                & \multirow{2}{*}{\sabr}              & $1\times$   & \textbf{98.82} & 93.38          \\
                                                         &                                                &                                     & $4\times$   & 98.48          & \textbf{93.85} \\

            \cmidrule(lr){1-3}
            \multirow{14.5}{*}{\cifar}                   & \multirow{7}{*}{$\frac{2}{255}$}               & \multirow{2}{*}{\ibp}               & $1\times$   & 67.93          & 55.85          \\
                                                         &                                                &                                     & $2\times$   & 68.06          & 56.18          \\
            \cmidrule(lr){3-3}
                                                         &                                                & \multirow{2}{*}{\ibpr}              & $1\times$   & 78.43          & 60.87          \\
                                                         &                                                &                                     & $2\times$   & \textbf{80.46} & 62.03          \\
            \cmidrule(lr){3-3}
                                                         &                                                & \multirow{2}{*}{\sabr}              & $1\times$   & 79.24          & 62.84          \\
                                                         &                                                &                                     & $2\times$   & 79.89          & \textbf{63.28} \\

            \cmidrule(lr){2-3}
                                            & \multirow{4.5}{*}{$\frac{8}{255}$} & \multirow{2}{*}{\ibp}  & $1\times$   & 47.35          & 34.17          \\
                                            &                                    &                        & $2\times$   & 47.83          & 33.98          \\
            \cmidrule(lr){3-3}
                                            &                                    & \multirow{2}{*}{\sabr} & $1\times$   & 50.78          & 34.12          \\
                                            &                                    &                        & $2\times$   & \textbf{51.56} & \textbf{34.95} \\

            \cmidrule(lr){1-3}

             \multirow{6.5}{*}{TinyImageNet} & \multirow{6.5}{*}{$\frac{1}{255}$} & \multirow{3}{*}{\ibp}   & $0.5\times$ & 24.47          & 18.76          \\
                                                         &                                                &                                     & $1\times$   & 25.33          & 19.46          \\
                                                         &                                                &                                     & $2\times$   & 25.40          & 19.92          \\

            \cmidrule(lr){3-3}
                                                         &                                                &  \multirow{3}{*}{\sabr} & $0.5\times$ & 27.56          & 20.54          \\
                                                         &                                                &                                     & $1\times$   & 28.63          & 21.21          \\
                                                         &                                                &                                     & $2\times$   & \textbf{28.97} & \textbf{21.36}           \\

            \bottomrule
        \end{tabular}
    \end{adjustbox}
\end{wraptable}

Next, we study the interaction of network architecture and \ibp training.
We train CNNs with 3 to 13 layers on \cifar for $\eps = 2/255$ (\cref{fig:IBP_architecture_effect}, right), using training-set \ibp-certified accuracy to measure fit and thus the regularizing effect of tightness. Greater depth should increase capacity, reducing robust training loss and improving accuracy. Instead, accuracy peaks at 7 layers and then drops. Tightness explains this trend: it decreases rapidly with depth up to 7 layers, then increases, indicating stronger regularization and reduced accuracy. This agrees with the popularity of the 7-layer \cnns in certified training \citep{ShiWZYH21, MuellerEFV22}.

Continuing our study of architecture effects, we train networks with 0.5 to 16 times the width of a standard \cnns using \ibp training and visualize the resulting \ibp certified train set accuracy and tightness in \cref{fig:IBP_architecture_effect} (left). We observe that increasing capacity via width instead of depth yields a monotone although diminishing increase in accuracy as tightness decreases gradually. The different trends for width and depth agree well with our theoretical results, predicting that sufficient network width is essential for trained networks (see \cref{thm:ibp_reconstruction}). It can further be explained by the observation that increasing depth, at initialization, reduces tightness exponentially, while increasing width only reduces it polynomially.
Intuitively, this suggests that less regularization is required to offset the tightness penalty of increasing network width rather than depth.

As these experiments indicate that optimal architectures for \ibp-based training have only moderate depth but large width, we train wider versions of the popular \cnns using \ibp, \sabr, and \ibpr, showing results in \cref{tb:sabr_wide,fig:sota_width}. We observe that this width increase improves certified accuracy in all settings. We note that, while these improvements might seem marginal, they are of similar magnitude as multiple years of progress on certified training methods, see \cref{fig:sota_width} where \crownibp \citep{ZhangCXGSLBH20} and \mtlibp \citep{PalmaBDKSL23} (the previous SOTA on \mnist) are shown for reference.

\subsection{Certified Training Increases Tightness} \label{sec:training_method}
To assess how different training methods affect tightness, we train a \cnnt on \cifar for a wide range of perturbation magnitudes ($\epsilon \in [10^{-5}, 5 \cdot 10^{-2}]$) using \ibp, \pgd, and \sabr training and illustrate the resulting tightness and accuracies in \cref{fig:cnn3_method}.
\ibp optimizes a sound upper bound on the worst-case loss over the whole input region. \sabr propagates a small subregion, yielding an unsound but generally more precise approximation. \pgd instead trains on samples that approximately maximize the loss, without \ibp.
We observe that training with either \ibp-based method increases tightness with perturbation magnitude until networks become almost propagation invariant with $\tau=0.98$ (see \cref{fig:cnn3_method}, right). This confirms our theoretical results, showing that \ibp training increases tightness with $\eps$ (see \cref{thm:radius_tightness}).
In contrast, training with \pgd barely influences tightness. %
The regularization required for high tightness severely reduces standard accuracy (\cref{fig:cnn3_method}, left) and, for $\eps \leq 5 \cdot 10^{-3}$, certified accuracy. For $\eps \geq 10^{-2}$, increased tightness instead improves certifiability enough to yield higher certified accuracy.

\begin{figure}[!t]
\captionsetup{font=footnotesize,skip=4pt}
    \centering
    \vspace{-2mm}
    \begin{subfigure}[b]{.3\linewidth}
        \centering
        \includegraphics[width=\linewidth]{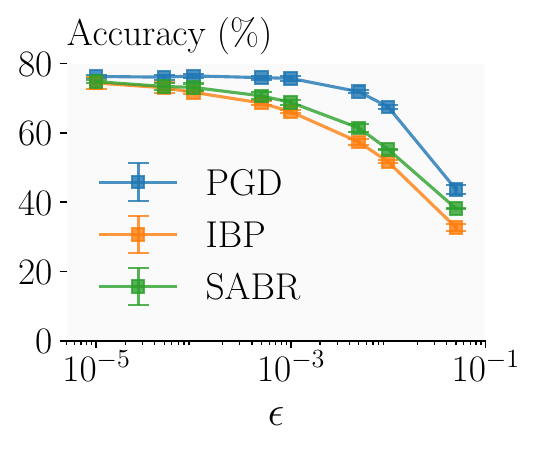}
    \end{subfigure}
    \hfil
    \begin{subfigure}[b]{.3\linewidth}
        \centering
        \includegraphics[width=\linewidth]{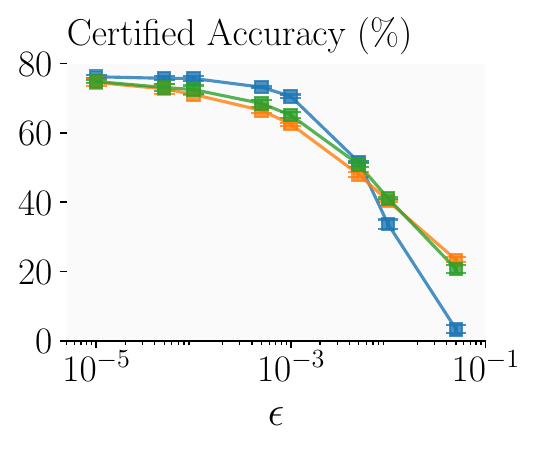}
    \end{subfigure}
    \hfil
    \begin{subfigure}[b]{.3\linewidth}
        \centering
        \includegraphics[width=\linewidth]{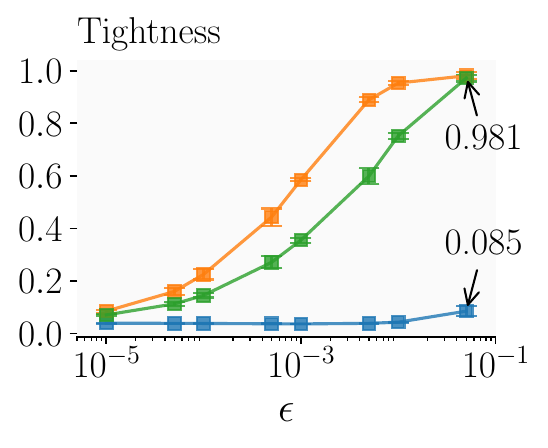}
    \end{subfigure}
    \caption{Tightness, standard, and certified accuracy for \cnnt on \cifar, depending on training method and perturbation magnitude $\epsilon$ used for training and evaluation.} \label{fig:cnn3_method}
\end{figure}

\begin{figure}[!t]
\captionsetup{font=footnotesize,skip=4pt}
    \begin{minipage}[t]{.64\linewidth}
        \centering
        \begin{subfigure}{.45\linewidth}
            \centering
            \includegraphics[width=\linewidth]{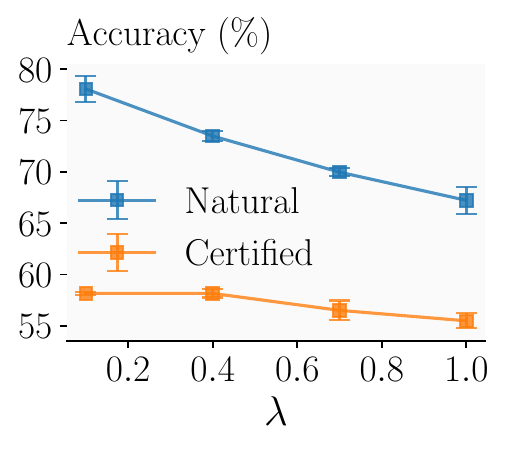}
        \end{subfigure}
        \hfil
        \begin{subfigure}{.45\linewidth}
            \centering
            \includegraphics[width=\linewidth]{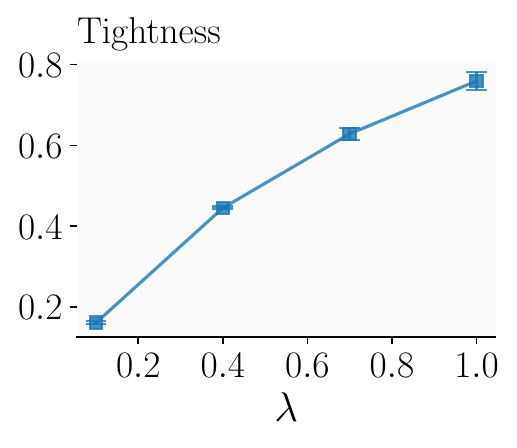}
        \end{subfigure}
        \caption{Accuracies and tightness of a \cnns for \cifar $\epsilon=\frac{2}{255}$ depending on regularization strength with \sabr.}
        \label{fig:SABR_lambda}
    \end{minipage}
    \hfil
    \begin{minipage}[t]{.32\linewidth}
        \centering
        \includegraphics[width=\linewidth]{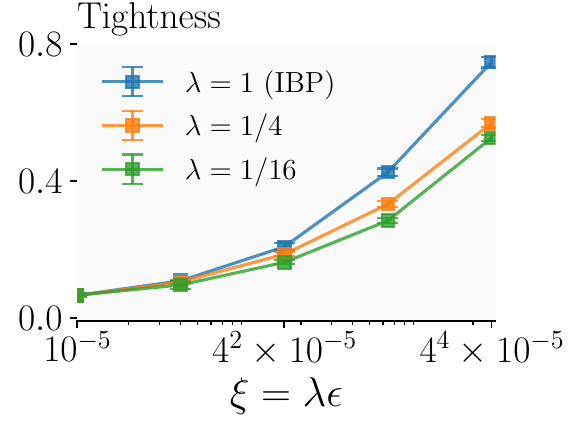}
        \caption{Tightness over propagation region size $\xi$ for \sabr.}
        \label{fig:SABR_lambda_tightness}
    \end{minipage}
\end{figure}

We vary \sabr's subselection ratio $\lambda$ to study the relationship between tightness and (certified) robustness. For fixed $\eps$, $\lambda$ controls the propagated region size, recovering \ibp at $\lambda=1$ and \pgd at $\lambda=0$. Decreasing $\lambda$ sharply reduces tightness and regularization, improving both natural and certified accuracy until tightness falls below $0.5$ at $\lambda=0.4$ (\cref{fig:SABR_lambda}). Other unsound certified training methods show similar trends when varying regularization (\cref{app:STAPS_regularization}).
In \cref{fig:SABR_lambda_tightness}, we vary $\epsilon$ for three values of $\lambda$ and plot tightness against propagation region size $\xi=\lambda\eps$ for a \cnnt on \cifar. Tightness is dominated by $\xi$ rather than $\epsilon$: training with IBP bounds increases tightness, although such high tightness and regularization are generally unnecessary for robustness. This helps explain \sabr's success and suggests developing certified training methods that reduce tightness while maintaining sufficient certifiability.

\begin{figure}[!htbp]
\centering
\begin{minipage}[t]{.46\textwidth}
\vspace{0pt}
    \captionsetup{font=footnotesize,skip=4pt}
    \captionof{table}{Tightness and accuracies for various training methods on \cifar.
    } \label{tb:colt}
    \footnotesize
    \setlength{\tabcolsep}{2pt}
    \begin{adjustbox}{max width=\linewidth}
    \begin{threeparttable}[b]
    \centering
        \begin{tabular}{@{}lcrrr@{}}
            \toprule
            Method                & $\eps$ & \multicolumn{1}{c}{Accuracy} & Tightness  & \multicolumn{1}{c}{Certified}  \\
            \midrule
            \multirow{2}{*}{PGD}  & 2/255  & 81.2     & 0.001         &  \multicolumn{1}{c}{-}   \\
                                  & 8/255  & 69.3     & 0.007         &  \multicolumn{1}{c}{-}   \\
            \cmidrule{1-2}
            \multirow{2}{*}{COLT} & 2/255  & 78.4\textsuperscript{*}      & 0.009         &  60.7\textsuperscript{*} \\
                                  & 8/255  & 51.7\textsuperscript{*}      & 0.057         &  26.7\textsuperscript{*}  \\
            \cmidrule{1-2}
            \multirow{2}{*}{IBP-R}  & 2/255  & 78.2\textsuperscript{*}      & 0.033         &  62.0\textsuperscript{*}   \\
                                  & 8/255  & 51.4\textsuperscript{*}     & 0.124 &  27.9\textsuperscript{*} \\
            \cmidrule{1-2}
            \multirow{2}{*}{SABR} & 2/255  & 75.6     & 0.182         & 57.7   \\
                                  & 8/255  & 48.2     & 0.950         & 31.2   \\
            \cmidrule{1-2}
            \multirow{2}{*}{IBP}  & 2/255  & 63.0     & 0.803         &  51.3  \\
                                  & 8/255  & 42.2     & 0.977         &  31.0  \\
            \bottomrule
        \end{tabular}
        \begin{tablenotes}
            \item [*] Literature result.
          \end{tablenotes}
    \end{threeparttable}
\end{adjustbox}
\end{minipage}%
\hspace{.06\textwidth}%
\begin{minipage}[t]{.33\textwidth}
\vspace{0pt}
\captionsetup{font=footnotesize,skip=3pt}
    \centering
    \begin{minipage}{0.80\linewidth}
        \centering
        \includegraphics[width=\linewidth]{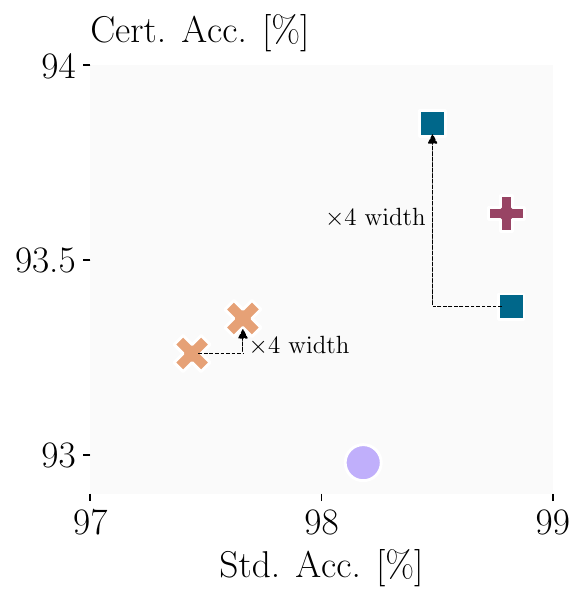}
    \end{minipage}\par
    \begin{minipage}{0.90\linewidth}
        \centering
        \includegraphics[width=\linewidth]{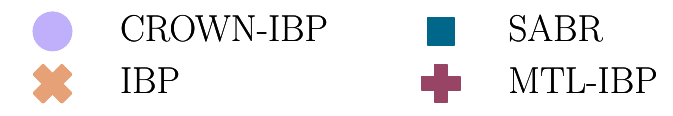}
    \end{minipage}\par
    \caption{Effect of a $4$-fold width increase on certified and standard accuracy for \mnist at $\epsilon=0.3$.}
    \label{fig:sota_width}
\end{minipage}
\end{figure}
\Cref{tb:colt} compares tightness, certified, and standard accuracy on the 4-layer CNN used by \colt and \ibpr. \colt uses no \ibp bounds, while \ibpr uses them only as a weak regularizer. Tightness and accuracy have opposite orderings, highlighting the accuracy penalty of tightness regularization. \colt and \ibpr increase tightness much less than \sabr or \ibp, yet remain an order of magnitude tighter than \pgd, suggesting that slightly increased tightness may benefit certified robustness. Consistently, the more strongly regularizing \sabr outperforms \ibpr at larger $\epsilon$ but underperforms at smaller $\epsilon$.

\vspace{-1mm}
\section{Related Work}
\vspace{-1mm}

\citet{BaaderMV20} show that continuous functions can be approximated by IBP-certifiable ReLU networks up to arbitrary precision. \citet{WangAPJ22} extend this result to more activation functions and characterize the hardness of such a construction. %
\citet{WangSGH22} find that \ibp-training converges to a global optimum with high probability for sufficient width.
\citet{mirman2022the} show that functions with points of non-invertibility can not be precisely approximated with \ibp. 
\citet{Zhu00C22a} show that width is advantageous while depth is not for approximate average case robustness.

\citet{ShiWZYH21} define \emph{tightness} as the layerwise \ibpl size $\Delta=\overline{\vz}^\dagger-\underline{\vz}^\dagger$. Our ratio $\tau$ compares this size with the optimal box (\cref{def:relu_tightness}), measuring approximation quality relative to the ground truth. This distinction yields different insights: scaling network weights leaves $\tau$ unchanged but scales $\Delta$ proportionally.

\citet{WuCCHG21} study the relation between empirical adversarial robustness and network width. They observe that in this setting, increased width actually hurts perturbation stability and thus potentially empirical robustness while improving natural accuracy.
In contrast, we have shown theoretically and empirically that width is beneficial for certified robustness when training with \ibp-based methods.

\vspace{-1mm}
\section{Conclusion}
\vspace{-1mm}

We investigate the mechanisms and trade-offs behind the recent, surprising dominance of \ibp-based certified training. Our propagation tightness metric relates \ibp to optimal \boxd bounds, predicting how architecture affects deep linear networks at initialization and after training. Experiments confirm these predictions for ReLU networks: wider networks improve state-of-the-art methods, while deeper networks do not.
\ibp-based training increases tightness according to the propagated region size, imposing strong regularization. This helps explain recent certified training methods' success. Together with our tightness metric, these insights may guide new methods that balance certifiability against the resulting (over-)regularization.

\par
\message{^^JLASTBODYPAGE \thepage^^J}
\clearpage
\setlength{\textfloatsep}{\savedtextfloatsep}
\setlength{\intextsep}{\savedintextsep}
\section*{Reproducibility Statement}
We publish our code, trained models, and detailed instructions on how to reproduce our results at \url{https://github.com/eth-sri/ibp-propagation-tightness}. 
Additionally, we provide detailed descriptions of all hyper-parameter choices, data sets, and preprocessing steps in \cref{app:exp_details}.

\section*{Acknowledgements}
We would like to thank our anonymous reviewers for their constructive comments and insightful questions.

This work has been done as part of the EU grant ELSA (European Lighthouse on Secure and Safe AI, grant agreement no. 101070617) and the SERI grant SAFEAI (Certified Safe, Fair and Robust Artificial Intelligence, contract no. MB22.00088). Views and opinions expressed are however those of the authors only and do not necessarily reflect those of the European Union or European Commission. Neither the European Union nor the European Commission can be held responsible for them. 

The work has received funding from the Swiss State Secretariat for Education, Research and Innovation (SERI).

\section*{Proof Formalization}
All theoretical results and proofs are formalized in Lean using Paper Formalizer~\citep{mao2026paperformalizer}, and are available at \url{https://github.com/AlgebraLoveme/ibp-theory-formalization}.

\bibliography{references}

\begin{thebibliography}{46}
\providecommand{\natexlab}[1]{#1}
\providecommand{\url}[1]{\texttt{#1}}
\expandafter\ifx\csname urlstyle\endcsname\relax
  \providecommand{\doi}[1]{doi: #1}\else
  \providecommand{\doi}{doi: \begingroup \urlstyle{rm}\Url}\fi

\bibitem[Baader et~al.(2020)Baader, Mirman, and Vechev]{BaaderMV20}
Maximilian Baader, Matthew Mirman, and Martin~T. Vechev.
\newblock Universal approximation with certified networks.
\newblock In \emph{Proc. of ICLR}, 2020.

\bibitem[Balunovic \& Vechev(2020)Balunovic and
  Vechev]{balunovic2020Adversarial}
Mislav Balunovic and Martin~T. Vechev.
\newblock Adversarial training and provable defenses: Bridging the gap.
\newblock In \emph{Proc. of ICLR}, 2020.

\bibitem[Biggio et~al.(2013)Biggio, Corona, Maiorca, Nelson, Srndic, Laskov,
  Giacinto, and Roli]{BiggioCMNSLGR13}
Battista Biggio, Igino Corona, Davide Maiorca, Blaine Nelson, Nedim Srndic,
  Pavel Laskov, Giorgio Giacinto, and Fabio Roli.
\newblock Evasion attacks against machine learning at test time.
\newblock In \emph{Proc of ECML PKDD}, volume 8190, 2013.
\newblock \doi{10.1007/978-3-642-40994-3\_25}.

\bibitem[Brix et~al.(2023)Brix, M{\"{u}}ller, Bak, Johnson, and
  Liu]{BrixMBJL22}
Christopher Brix, Mark~Niklas M{\"{u}}ller, Stanley Bak, Taylor~T. Johnson, and
  Changliu Liu.
\newblock First three years of the international verification of neural
  networks competition {(VNN-COMP)}.
\newblock \emph{CoRR}, abs/2301.05815, 2023.
\newblock \doi{10.48550/arXiv.2301.05815}.

\bibitem[Bunel et~al.(2020)Bunel, Lu, Turkaslan, Torr, Kohli, and
  Kumar]{BunelLTTKK20}
Rudy Bunel, Jingyue Lu, Ilker Turkaslan, Philip H.~S. Torr, Pushmeet Kohli, and
  M.~Pawan Kumar.
\newblock Branch and bound for piecewise linear neural network verification.
\newblock \emph{J. Mach. Learn. Res.}, 21, 2020.

\bibitem[Chorowski \& Zurada(2014)Chorowski and Zurada]{ChorowskiZ14}
Jan Chorowski and Jacek~M Zurada.
\newblock Learning understandable neural networks with nonnegative weight
  constraints.
\newblock \emph{IEEE transactions on neural networks and learning systems},
  26\penalty0 (1), 2014.
\newblock \doi{10.1109/TNNLS.2014.2310059}.

\bibitem[Cook(1957)]{cook1957rational}
JM~Cook.
\newblock Rational formulae for the production of a spherically symmetric
  probability distribution.
\newblock \emph{Mathematics of Computation}, 11\penalty0 (58), 1957.

\bibitem[Croce \& Hein(2020)Croce and Hein]{Croce020a}
Francesco Croce and Matthias Hein.
\newblock Reliable evaluation of adversarial robustness with an ensemble of
  diverse parameter-free attacks.
\newblock In \emph{Proc. of ICML}, volume 119, 2020.

\bibitem[Ferrari et~al.(2022)Ferrari, M{\"{u}}ller, Jovanovic, and
  Vechev]{FerrariMJV22}
Claudio Ferrari, Mark~Niklas M{\"{u}}ller, Nikola Jovanovic, and Martin~T.
  Vechev.
\newblock Complete verification via multi-neuron relaxation guided
  branch-and-bound.
\newblock In \emph{Proc. of ICLR}, 2022.

\bibitem[Glorot \& Bengio(2010)Glorot and Bengio]{GlorotB10}
Xavier Glorot and Yoshua Bengio.
\newblock Understanding the difficulty of training deep feedforward neural
  networks.
\newblock In \emph{Proc. of AISTATS}, volume~9, 2010.

\bibitem[Gowal et~al.(2018)Gowal, Dvijotham, Stanforth, Bunel, Qin, Uesato,
  Arandjelovic, Mann, and Kohli]{GowalIBP2018}
Sven Gowal, Krishnamurthy Dvijotham, Robert Stanforth, Rudy Bunel, Chongli Qin,
  Jonathan Uesato, Relja Arandjelovic, Timothy~A. Mann, and Pushmeet Kohli.
\newblock On the effectiveness of interval bound propagation for training
  verifiably robust models.
\newblock \emph{ArXiv preprint}, abs/1810.12715, 2018.

\bibitem[He et~al.(2015)He, Zhang, Ren, and Sun]{HeZRS15}
Kaiming He, Xiangyu Zhang, Shaoqing Ren, and Jian Sun.
\newblock Delving deep into rectifiers: Surpassing human-level performance on
  imagenet classification.
\newblock In \emph{Proc. of ICCV}, 2015.
\newblock \doi{10.1109/ICCV.2015.123}.

\bibitem[Ji \& Telgarsky(2019)Ji and Telgarsky]{JiT19}
Ziwei Ji and Matus Telgarsky.
\newblock Gradient descent aligns the layers of deep linear networks.
\newblock In \emph{Proc. of ICLR}, 2019.

\bibitem[Jovanović et~al.(2022)Jovanović, Balunović, Baader, and
  Vechev]{jovanovic2022paradox}
Nikola Jovanović, Mislav Balunović, Maximilian Baader, and Martin Vechev.
\newblock On the paradox of certified training.
\newblock In \emph{Transactions on Machine Learning Research}, 2022.

\bibitem[Krizhevsky et~al.(2009)Krizhevsky, Hinton,
  et~al.]{krizhevsky2009learning}
Alex Krizhevsky, Geoffrey Hinton, et~al.
\newblock Learning multiple layers of features from tiny images.
\newblock 2009.

\bibitem[LeCun et~al.(2010)LeCun, Cortes, and Burges]{lecun2010mnist}
Yann LeCun, Corinna Cortes, and CJ~Burges.
\newblock Mnist handwritten digit database.
\newblock \emph{ATT Labs [Online]. Available:
  http://yann.lecun.com/exdb/mnist}, 2, 2010.

\bibitem[Lin et~al.(2022)Lin, Ivanov, Weimer, Sokolsky, and Lee]{LinIWSL22}
Vivian Lin, Radoslav Ivanov, James Weimer, Oleg Sokolsky, and Insup Lee.
\newblock {T4V:} exploring neural network architectures that improve the
  scalability of neural network verification.
\newblock In \emph{Principles of Systems Design - Essays Dedicated to Thomas A.
  Henzinger on the Occasion of His 60th Birthday}, volume 13660, 2022.
\newblock \doi{10.1007/978-3-031-22337-2\_28}.

\bibitem[Mao et~al.(2023)Mao, M{\"{u}}ller, Fischer, and Vechev]{MaoMFV2023}
Yuhao Mao, Mark~Niklas M{\"{u}}ller, Marc Fischer, and Martin Vechev.
\newblock Taps: Connecting certified and adversarial training, 2023.

\bibitem[Mao et~al.(2026)Mao, Sun, Zhang, and Vechev]{mao2026paperformalizer}
Yuhao Mao, Chenhao Sun, Yani Zhang, and Martin Vechev.
\newblock Paper formalization for efficient and rigorous proof review, 2026.
\newblock URL
  \url{https://files.sri.inf.ethz.ch/website/papers/yuhao2026paper_formalizer.pdf}.
\newblock Manuscript under review.

\bibitem[Marsaglia(1972)]{marsaglia1972choosing}
George Marsaglia.
\newblock Choosing a point from the surface of a sphere.
\newblock \emph{The Annals of Mathematical Statistics}, 43\penalty0 (2), 1972.

\bibitem[Mirman et~al.(2018)Mirman, Gehr, and Vechev]{MirmanGV18}
Matthew Mirman, Timon Gehr, and Martin~T. Vechev.
\newblock Differentiable abstract interpretation for provably robust neural
  networks.
\newblock In \emph{Proc. of ICML}, volume~80, 2018.

\bibitem[Mirman et~al.(2022)Mirman, Baader, and Vechev]{mirman2022the}
Matthew~B Mirman, Maximilian Baader, and Martin Vechev.
\newblock The fundamental limits of neural networks for interval certified
  robustness.
\newblock \emph{Transactions on Machine Learning Research}, 2022.
\newblock ISSN 2835-8856.

\bibitem[M{\"{u}}ller et~al.(2022{\natexlab{a}})M{\"{u}}ller, Brix, Bak, Liu,
  and Johnson]{MuellerBBLJ22}
Mark~Niklas M{\"{u}}ller, Christopher Brix, Stanley Bak, Changliu Liu, and
  Taylor~T. Johnson.
\newblock The third international verification of neural networks competition
  {(VNN-COMP} 2022): Summary and results.
\newblock \emph{CoRR}, abs/2212.10376, 2022{\natexlab{a}}.
\newblock \doi{10.48550/arXiv.2212.10376}.

\bibitem[M{\"{u}}ller et~al.(2022{\natexlab{b}})M{\"{u}}ller, Eckert, Fischer,
  and Vechev]{MuellerEFV22}
Mark~Niklas M{\"{u}}ller, Franziska Eckert, Marc Fischer, and Martin~T. Vechev.
\newblock Certified training: Small boxes are all you need.
\newblock \emph{CoRR}, abs/2210.04871, 2022{\natexlab{b}}.

\bibitem[M{\"{u}}ller et~al.(2022{\natexlab{c}})M{\"{u}}ller, Makarchuk, Singh,
  P{\"{u}}schel, and Vechev]{MullerMSPV22}
Mark~Niklas M{\"{u}}ller, Gleb Makarchuk, Gagandeep Singh, Markus
  P{\"{u}}schel, and Martin~T. Vechev.
\newblock {PRIMA:} general and precise neural network certification via
  scalable convex hull approximations.
\newblock \emph{Proc. {ACM} Program. Lang.}, 6\penalty0 ({POPL}),
  2022{\natexlab{c}}.
\newblock \doi{10.1145/3498704}.

\bibitem[Palma et~al.(2022)Palma, Bunel, Dvijotham, Kumar, and
  Stanforth]{PalmaIBPR22}
Alessandro~De Palma, Rudy Bunel, Krishnamurthy Dvijotham, M.~Pawan Kumar, and
  Robert Stanforth.
\newblock {IBP} regularization for verified adversarial robustness via
  branch-and-bound.
\newblock \emph{ArXiv preprint}, abs/2206.14772, 2022.

\bibitem[Palma et~al.(2023)Palma, Bunel, Dvijotham, Kumar, Stanforth, and
  Lomuscio]{PalmaBDKSL23}
Alessandro~De Palma, Rudy Bunel, Krishnamurthy Dvijotham, M.~Pawan Kumar,
  Robert Stanforth, and Alessio Lomuscio.
\newblock Expressive losses for verified robustness via convex combinations.
\newblock \emph{CoRR}, abs/2305.13991, 2023.
\newblock \doi{10.48550/arXiv.2305.13991}.
\newblock URL \url{https://doi.org/10.48550/arXiv.2305.13991}.

\bibitem[Pope et~al.(2021)Pope, Zhu, Abdelkader, Goldblum, and
  Goldstein]{PopeZAGG21}
Phillip Pope, Chen Zhu, Ahmed Abdelkader, Micah Goldblum, and Tom Goldstein.
\newblock The intrinsic dimension of images and its impact on learning.
\newblock In \emph{Proc. of ICLR}, 2021.

\bibitem[Ribeiro et~al.(2016)Ribeiro, Singh, and Guestrin]{Ribeiro0G16}
Marco~T{\'{u}}lio Ribeiro, Sameer Singh, and Carlos Guestrin.
\newblock "why should {I} trust you?": Explaining the predictions of any
  classifier.
\newblock In \emph{Proceedings of the 22nd {ACM} {SIGKDD} International
  Conference on Knowledge Discovery and Data Mining, San Francisco, CA, USA,
  August 13-17, 2016}, 2016.
\newblock \doi{10.1145/2939672.2939778}.

\bibitem[Saxe et~al.(2014)Saxe, McClelland, and Ganguli]{SaxeMG13}
Andrew~M. Saxe, James~L. McClelland, and Surya Ganguli.
\newblock Exact solutions to the nonlinear dynamics of learning in deep linear
  neural networks.
\newblock In \emph{Proc. of ICLR}, 2014.

\bibitem[Shi et~al.(2021)Shi, Wang, Zhang, Yi, and Hsieh]{ShiWZYH21}
Zhouxing Shi, Yihan Wang, Huan Zhang, Jinfeng Yi, and Cho{-}Jui Hsieh.
\newblock Fast certified robust training with short warmup.
\newblock In \emph{Proc. of NeurIPS}, 2021.

\bibitem[Singh et~al.(2018)Singh, Gehr, Mirman, P{\"{u}}schel, and
  Vechev]{SinghGMPV18}
Gagandeep Singh, Timon Gehr, Matthew Mirman, Markus P{\"{u}}schel, and
  Martin~T. Vechev.
\newblock Fast and effective robustness certification.
\newblock In \emph{Proc. of NeurIPS}, 2018.

\bibitem[Singh et~al.(2019{\natexlab{a}})Singh, Ganvir, P{\"{u}}schel, and
  Vechev]{SinghGPV19B}
Gagandeep Singh, Rupanshu Ganvir, Markus P{\"{u}}schel, and Martin~T. Vechev.
\newblock Beyond the single neuron convex barrier for neural network
  certification.
\newblock In \emph{Proc. of NeurIPS}, 2019{\natexlab{a}}.

\bibitem[Singh et~al.(2019{\natexlab{b}})Singh, Gehr, P{\"{u}}schel, and
  Vechev]{SinghGPV19}
Gagandeep Singh, Timon Gehr, Markus P{\"{u}}schel, and Martin~T. Vechev.
\newblock An abstract domain for certifying neural networks.
\newblock \emph{Proc. {ACM} Program. Lang.}, 3\penalty0 ({POPL}),
  2019{\natexlab{b}}.

\bibitem[Szegedy et~al.(2014)Szegedy, Zaremba, Sutskever, Bruna, Erhan,
  Goodfellow, and Fergus]{SzegedyZSBEGF13}
Christian Szegedy, Wojciech Zaremba, Ilya Sutskever, Joan Bruna, Dumitru Erhan,
  Ian~J. Goodfellow, and Rob Fergus.
\newblock Intriguing properties of neural networks.
\newblock In \emph{Proc. of ICLR}, 2014.

\bibitem[Tjeng et~al.(2019)Tjeng, Xiao, and Tedrake]{TjengXT19}
Vincent Tjeng, Kai~Y. Xiao, and Russ Tedrake.
\newblock Evaluating robustness of neural networks with mixed integer
  programming.
\newblock In \emph{Proc. of ICLR}, 2019.

\bibitem[Tram{\`{e}}r et~al.(2020)Tram{\`{e}}r, Carlini, Brendel, and
  Madry]{TramerCBM20}
Florian Tram{\`{e}}r, Nicholas Carlini, Wieland Brendel, and Aleksander Madry.
\newblock On adaptive attacks to adversarial example defenses.
\newblock In \emph{Proc. of NeurIPS}, 2020.

\bibitem[Tsipras et~al.(2019)Tsipras, Santurkar, Engstrom, Turner, and
  Madry]{TsiprasSETM19}
Dimitris Tsipras, Shibani Santurkar, Logan Engstrom, Alexander Turner, and
  Aleksander Madry.
\newblock Robustness may be at odds with accuracy.
\newblock In \emph{Proc. of ICLR}, 2019.

\bibitem[Wang et~al.(2022{\natexlab{a}})Wang, Shi, Gu, and Hsieh]{WangSGH22}
Yihan Wang, Zhouxing Shi, Quanquan Gu, and Cho{-}Jui Hsieh.
\newblock On the convergence of certified robust training with interval bound
  propagation.
\newblock In \emph{Proc. of ICLR}, 2022{\natexlab{a}}.

\bibitem[Wang et~al.(2022{\natexlab{b}})Wang, Albarghouthi, Prakriya, and
  Jha]{WangAPJ22}
Zi~Wang, Aws Albarghouthi, Gautam Prakriya, and Somesh Jha.
\newblock Interval universal approximation for neural networks.
\newblock \emph{Proc. {ACM} Program. Lang.}, 6\penalty0 ({POPL}),
  2022{\natexlab{b}}.
\newblock \doi{10.1145/3498675}.

\bibitem[Wong et~al.(2018)Wong, Schmidt, Metzen, and Kolter]{WongSMK18}
Eric Wong, Frank~R. Schmidt, Jan~Hendrik Metzen, and J.~Zico Kolter.
\newblock Scaling provable adversarial defenses.
\newblock In \emph{Proc. of NeurIPS}, 2018.

\bibitem[Wu et~al.(2021)Wu, Chen, Cai, He, and Gu]{WuCCHG21}
Boxi Wu, Jinghui Chen, Deng Cai, Xiaofei He, and Quanquan Gu.
\newblock Do wider neural networks really help adversarial robustness?
\newblock In Marc'Aurelio Ranzato, Alina Beygelzimer, Yann~N. Dauphin, Percy
  Liang, and Jennifer~Wortman Vaughan (eds.), \emph{Advances in Neural
  Information Processing Systems 34: Annual Conference on Neural Information
  Processing Systems 2021, NeurIPS 2021, December 6-14, 2021, virtual}, pp.\
  7054--7067, 2021.
\newblock URL
  \url{https://proceedings.neurips.cc/paper/2021/hash/3937230de3c8041e4da6ac3246a888e8-Abstract.html}.

\bibitem[Wu et~al.(2019)Wu, Wang, and Ma]{WuWM19}
Lei Wu, Qingcan Wang, and Chao Ma.
\newblock Global convergence of gradient descent for deep linear residual
  networks.
\newblock In \emph{Proc. of NeurIPS}, 2019.

\bibitem[Zhang et~al.(2020)Zhang, Chen, Xiao, Gowal, Stanforth, Li, Boning, and
  Hsieh]{ZhangCXGSLBH20}
Huan Zhang, Hongge Chen, Chaowei Xiao, Sven Gowal, Robert Stanforth, Bo~Li,
  Duane~S. Boning, and Cho{-}Jui Hsieh.
\newblock Towards stable and efficient training of verifiably robust neural
  networks.
\newblock In \emph{Proc. of ICLR}, 2020.

\bibitem[Zhang et~al.(2022)Zhang, Wang, Xu, Li, Li, Jana, Hsieh, and
  Kolter]{ZhangWXLLJ22}
Huan Zhang, Shiqi Wang, Kaidi Xu, Linyi Li, Bo~Li, Suman Jana, Cho{-}Jui Hsieh,
  and J.~Zico Kolter.
\newblock General cutting planes for bound-propagation-based neural network
  verification.
\newblock \emph{ArXiv preprint}, abs/2208.05740, 2022.

\bibitem[Zhu et~al.(2022)Zhu, Liu, Chrysos, and Cevher]{Zhu00C22a}
Zhenyu Zhu, Fanghui Liu, Grigorios Chrysos, and Volkan Cevher.
\newblock Robustness in deep learning: The good (width), the bad (depth), and
  the ugly (initialization).
\newblock In \emph{NeurIPS}, 2022.
\newblock URL
  \url{http://papers.nips.cc/paper\_files/paper/2022/hash/ea5a63f7ddb82e58623693fd1f4933f7-Abstract-Conference.html}.

\end{thebibliography}
\bibliographystyle{iclr2024_conference}

\message{^^JLASTREFERENCESPAGE \thepage^^J}

\ifbool{includeappendix}{%
	\clearpage
	\appendix
	\section{Additional Theoretical Results} \label{app:theory}

Below we present a corollary, formalizing the intutions we provided in \cref{sec:pi_conditions}.
\begin{restatable}[]{cor}{chosable_signs}
    \label[corollary]{cor:linear_param}
    Assume all elements of $\mW^{(1)}$, $\mW^{(2)}$ and $\mW^{(2)} \mW^{(1)}$ are non-zero and $\mW^{(2)} \mW^{(1)}$ is propagation invariant. Then choosing the signs of one row and one column of $\mW^{(2)} \mW^{(1)}$ fixes all signs of $\mW^{(2)} \mW^{(1)}$.
\end{restatable}

\begin{proof}
    For notational reasons, we define $W:=W^{(2)} W^{(1)}$. Without loss of generality, assume we know the signs of the first row and the first column, \ie, $W_{1, \cdot}$ and $W_{\cdot, 1}$. We prove via a construction of the signs of all elements. The construction is given by the following: whenever $\exists i, j$, such that we know the sign of $W_{i,j}$, $W_{i, j+1}$ and $W_{i+1, j}$, we fix the sign of $W_{i+1, j+1}$ to be positive if there are an odd number of positive elements among $W_{i,j}$, $W_{i, j+1}$ and $W_{i+1, j}$, otherwise negative.

    By \cref{thm:invariant-reg}, propagation invariance requires us to fix the sign of the last element in the $W_{i:i+1, j:j+1}$ block in this way. We only need to prove that when this process terminates, we fix the signs of all elements. We show this via recursion.

    When $i=1$ and $j=1$, we have known the signs of $W_{i,j}$, $W_{i, j+1}$ and $W_{i+1, j}$, thus the sign of $W_{i+1, j+1}$ is fixed. Continuing towards the right, we gradually fix the sign of $W_{2, j+1}$ for $j=1, \dots, d-1$. Continuing downwards, we gradually fix the sign of $W_{i+1, 2}$ for $i=1, \dots, d-1$. Therefore, all signs of the elements of the second row and the second column are fixed. By recursion, we would finally fix all the rows and the columns, thus the whole matrix.
\end{proof}

We extend our result in \cref{sec:initialization} to two-layer ReLU networks. For a fixed nonzero input, Gaussian row symmetry makes each ReLU active with probability $1/2$ and makes its activation indicator independent of the absolute and squared row coordinates.

\begin{restatable}[]{cor}{2_layer_relu_init}
    \label[corollary]{cor:2_layer_relu_init}
    Assume $X\overset{d}{=}-X$ and $\Pr(X=0)=0$. Consider a two-layer ReLU network $\vf = \mW^{(2)}\relu(\mW^{(1)}\vx)$ with independent centered Gaussian weights across entries and layers, positive layer variances, and initialization independent of $X$. Its ratio of expected local radius-growth rates satisfies $\tau^\prime(d_1)\sim\sqrt{2}\,\tau(d_1)$ as the hidden width $d_1\to\infty$, where $\tau(d_1)$ is the corresponding DLN ratio of expected radii.
\end{restatable}

\begin{proof}
Let $n=d_1$ denote the hidden width and write $\tau'_n$ and $\tau_n$ for the ReLU and DLN tightness, respectively, to make their width dependence explicit.
Under the Gaussian initialization of \cref{sec:initialization}, the entries of $\mW^{(1)}\in\mathbb{R}^{n\times d_0}$ and $\mW^{(2)}\in\mathbb{R}^{d_2\times n}$ have distributions $\mathcal{N}(0,\sigma_1^2)$ and $\mathcal{N}(0,\sigma_2^2)$, with $\sigma_1,\sigma_2>0$.
The two layers and the input are independent.
Condition on a fixed input $\vx\ne0$. Write
$a_k=\mW^{(1)}_{k,\cdot}/\sigma_1$,
$I_k=\mathds{1}_{\langle a_k,\vx\rangle>0}$, and
$D=\operatorname{diag}(I_1,\ldots,I_n)$. Each preactivation is a
nondegenerate Gaussian, so all preactivations are nonzero almost surely.
Consequently, the activation pattern is constant on a sufficiently small box
around $\vx$. The optimal and layer-wise radius-growth vectors are

\begin{equation*}
\begin{aligned}
\mathbf{r}^{\prime *}
&=|\mW^{(2)}D\mW^{(1)}|\bm{1},
&\mathbf{r}^{\prime\dagger}
&=|\mW^{(2)}|D|\mW^{(1)}|\bm{1},\\
\mathbf{r}^*
&=|\mW^{(2)}\mW^{(1)}|\bm{1},
&\mathbf{r}^\dagger
&=|\mW^{(2)}||\mW^{(1)}|\bm{1}.
\end{aligned}
\end{equation*}

Here the primed vectors correspond to the ReLU network. For any output
coordinate $i$,

\begin{equation*}
\tau'_n=\frac{\E r_i^{\prime *}}
{\E r_i^{\prime\dagger}},
\qquad
\tau_n=\frac{\E r_i^*}{\E r_i^\dagger}.
\end{equation*}

The row transformation $a_k\mapsto-a_k$ preserves its distribution and
flips $I_k$. Therefore, for any integrable even function $\phi$,

\begin{equation}
\E\bigl[I_k\phi(a_k)\bigr]
=\frac12\E\phi(a_k).
\label{eq:relu_init_even_symmetry}
\end{equation}

In particular, $I_k$ is independent of the squared row coordinates and
has distribution $\operatorname{Bernoulli}(1/2)$. Independence of the rows
then gives

\begin{equation*}
K:=\sum_{k=1}^n I_k\sim\operatorname{Binomial}(n,1/2).
\end{equation*}

Using \cref{eq:relu_init_even_symmetry} with $\phi(a_k)=|a_{kj}|$, independence of the second layer,
and $\E|\mathcal N(0,1)|=\sqrt{2/\pi}$, we obtain

\begin{equation}
\begin{aligned}
\E r_i^{\prime\dagger}
&=\sum_{j=1}^{d_0}\sum_{k=1}^n
\E|W^{(2)}_{ik}|\,
\E\bigl[I_k|W^{(1)}_{kj}|\bigr]\\
&=\frac{\sigma_1\sigma_2d_0n}{\pi}
=\frac12\E r_i^\dagger.
\end{aligned}
\label{eq:relu_init_layerwise_radius}
\end{equation}

For the optimal radius, condition further on the first layer. Then

\begin{equation*}
\sum_{k=1}^n W^{(2)}_{ik}I_kW^{(1)}_{kj}
\;\Big|\;\mW^{(1)}
\sim\mathcal N\!\left(0,
\sigma_1^2\sigma_2^2\sum_{k=1}^n I_k a_{kj}^2\right).
\end{equation*}

Conditional on the activation indicators, the selected squared coordinates
remain independent $\chi_1^2$ variables by \cref{eq:relu_init_even_symmetry}. Hence
$\sum_k I_k a_{kj}^2\mid K\sim\chi_K^2$, with the sum equal to zero
when $K=0$. Define

\begin{equation*}
\mu(0)=0,
\qquad
\mu(k)=\E\sqrt{\chi_k^2}
=\frac{\sqrt2\,\Gamma((k+1)/2)}{\Gamma(k/2)}
\quad(k\ge1).
\end{equation*}

Taking the absolute Gaussian expectation and summing over input coordinates
gives

\begin{equation}
\E r_i^{\prime *}
=\sigma_1\sigma_2d_0\sqrt{\frac2\pi}\,\E\mu(K),
\qquad
\E r_i^*
=\sigma_1\sigma_2d_0\sqrt{\frac2\pi}\,\mu(n).
\label{eq:relu_init_optimal_radius}
\end{equation}

The expectations in \cref{eq:relu_init_layerwise_radius,eq:relu_init_optimal_radius} are independent of the conditioned input $\vx$,
so they also hold after averaging over $X$. They imply the exact
finite-width identity

\begin{equation}
\frac{\tau'_n}{\tau_n}
=\frac{2\E\mu(K)}{\mu(n)}
=\frac{2}{2^n\mu(n)}\sum_{k=0}^n\binom nk\mu(k).
\label{eq:relu_init_finite_width}
\end{equation}

It remains to evaluate its limit. Let $B_k$ be independent
$\operatorname{Bernoulli}(1/2)$ variables and $Z_k$ independent
standard Gaussians, independent of the $B_k$. Put
$S_n=\sum_{k=1}^n B_kZ_k^2$ and $T_n=\sum_{k=1}^n Z_k^2$.
Then

\begin{equation*}
\begin{gathered}
\E\sqrt{S_n}=\E\mu(K),
\qquad \E\sqrt{T_n}=\mu(n),\\
\E(S_n/n)=\frac12,
\qquad\operatorname{Var}(S_n/n)=\frac{5}{4n},\\
\E(T_n/n)=1,
\qquad\operatorname{Var}(T_n/n)=\frac2n.
\end{gathered}
\end{equation*}

For $u\ge0$ and $a>0$, rationalizing the difference gives
$|\sqrt u-\sqrt a|=\frac{|u-a|}{\sqrt u+\sqrt a}\le\frac{|u-a|}{\sqrt a}$,
since $\sqrt u+\sqrt a\ge\sqrt a>0$. For a nonnegative random variable $U$ with $\E U=a>0$ and finite variance, this bound and Cauchy--Schwarz give
\begin{equation*}
\left|\E\sqrt U-\sqrt a\right|
\le\E\left|\sqrt U-\sqrt a\right|
\le\frac{\E|U-a|}{\sqrt a}
\le\frac{\sqrt{\E(U-a)^2}\sqrt{\E 1^2}}{\sqrt a}
=\sqrt{\frac{\operatorname{Var}(U)}{a}}.
\end{equation*}
Applying this inequality first with $U=S_n/n$ and $a=1/2$, then with $U=T_n/n$ and $a=1$, yields

\begin{equation*}
\left|\frac{\E\mu(K)}{\sqrt n}-\frac1{\sqrt2}\right|
\le\sqrt{\frac{5}{2n}},
\qquad
\left|\frac{\mu(n)}{\sqrt n}-1\right|
\le\sqrt{\frac2n}.
\end{equation*}

Therefore,
\begin{equation*}
\frac{\tau'_n}{\tau_n}
=2\frac{\E\mu(K)/\sqrt n}{\mu(n)/\sqrt n}
\longrightarrow\sqrt2,
\qquad\text{hence}\qquad
\tau'_n\sim\sqrt2\,\tau_n.
\end{equation*}
\end{proof}

We perform a Monte-Carlo estimation of the ratio $\tau^\prime / \tau$ with a two-layer fully connected network and a two-layer convolutional network on \mnist. The estimation is $1.4167 \pm 0.0059$ and $1.4228 \pm 0.0368$, respectively, which is close to the theoretical value $\sqrt{2} \approx 1.4142$. This confirms the correctness of our theoretical analysis and its generalization even to convolutional networks which do not fully satisfy the assumption.

\section{Deferred Proofs} \label{app:proofs}

\paragraph{Proof of \cref{lem:opt_box_bound}}

Here we prove \cref{lem:opt_box_bound}, restated below for convenience.

\exactcert*

\begin{proof}
For each $i\ne t$, continuity on the compact box $\B^{\bm{\epsilon}}(\vx)$ gives a maximizing input $\vx_i^*\in\argmax_{\vx'\in\B^{\bm{\epsilon}}(\vx)} y_i^\Delta(\vx')$.
The optimal upper bound is therefore $\obar{y}^{\Delta^*}_i=y_i^\Delta(\vx_i^*)$.
If $t$ is the unique winning class throughout the box, every logit difference is strictly negative, including at its maximizing input, so $\obar{y}^{\Delta^*}_i<0$.
Conversely, if every optimal upper bound is negative, then $y_i^\Delta(\vx')\le\obar{y}^{\Delta^*}_i<0$ for all $\vx'$ in the box and all $i\ne t$, so $t$ is the unique winning class throughout the box.
\end{proof}

\paragraph{Proof of \cref{lem:box_size_L}}

We first prove \cref{lem:box_size_L} for a 2-layer DLN as \cref{lem:box_size_2}.

\begin{restatable}[]{lem}{boxprop2}
\label[lemma]{lem:box_size_2}
    For a two-layer DLN $\vf = \mW^{(2)} \mW^{(1)}$, $(\ubopt - \lbopt) /2 =  \left| W^{(2)}W^{(1)} \right| \bm{\epsilon}$ and $(\ubbox - \lbbox) /2 = \left| W^{(2)} \right| \left| W^{(1)} \right| \bm{\epsilon}$. In addition, $\ibpb$ and $\ibpl$ have the same center $\vf(\vx)$.
\end{restatable}

\begin{proof}
    First, assume $W^{(1)} \in \mathcal{R}^{d_1\times d_0}$, $W^{(2)} \in \mathcal{R}^{d_2\times d_1}$ and $B_i = [-1,1]^{d_i}$ for $i=0,1,2$, where $d_i \in \mathcal{Z}_+$ are some positive integers. The input box can be represented as $\diag(\epsilon_0) B_0 + b$ for $\epsilon_0 = \epsilon$.

    For a single linear layer, the box propagation yields
    \begin{align}
         \ibpm(W^{(1)}(\diag(\epsilon_0) B_0 + b)) \nonumber                                      
         & = \ibpm(W^{(1)} \diag(\epsilon_0) B_0) + W^{(1)} b \nonumber                                 \\
         & = \diag\left(\sum_{j=1}^{d_0} |W^{(1)}_{i,j}| \epsilon_0[j]\right) B_1 + W^{(1)} b \nonumber \\
         & := \diag(\epsilon_1) B_1 + W^{(1)} b. \label{eq:pf1-1}
    \end{align}
    Applying \cref{eq:pf1-1} iteratively, we get the explicit formula of layer-wise propagation for two-layer linear network:
    \begin{align}
         & \ibpm(W^{(2)} \ibpm(W^{(1)} (\diag(\epsilon_0) B_0 + b))) \nonumber                                                         \\
         & =\ibpm\left(W^{(2)}(\diag(\epsilon_1) B_1 + W^{(1)} b)\right) \nonumber                                                    \\
         & = \diag\left(\sum_{k=1}^{d_1} |W^{(2)}_{i,k}| \epsilon_1[k]\right) B_2 + W^{(2)} W^{(1)} b \nonumber                           \\
         & = \diag\left(\sum_{j=1}^{d_0} \epsilon_0[j] \left(\sum_{k=1}^{d_1} |W^{(2)}_{i,k} W^{(1)}_{k,j}|\right)\right)B_2 + W^{(2)} W^{(1)} b. \label{eq:pf1-2}
    \end{align}
    Applying \cref{eq:pf1-1} on $W:=W^{(2)} W^{(1)}$, we get the explicit formula of the tightest box:
    \begin{align}
         & \ibpm(W^{(2)} W^{(1)}(\diag(\epsilon_0) B_0 + b)) \nonumber                                                                 \\
         & = \diag\left(\sum_{j=1}^{d_0} | (W^{(2)} W^{(1)})_{i,j} | \epsilon_0[j]\right) B_2 + W^{(2)} W^{(1)} b\nonumber                     \\
         & = \diag\left(\sum_{j=1}^{d_0} \epsilon_0[j]\left| \sum_{k=1}^{d_1} W^{(2)}_{i,k} W^{(1)}_{k,j} \right| \right) B_2 + W^{(2)} W^{(1)} b. \label{eq:pf1-3}
    \end{align}
\end{proof}

Now, we use induction and \cref{lem:box_size_2} to prove \cref{lem:box_size_L}, restated below for convenience. The key insight is that a multi-layer DLN is equivalent to a single-layer linear network. Thus, we can group layers together and view general DLNs as two-layer DLNs.

\boxprop*

\begin{proof}
    For $L=1$, the result follows directly from \cref{eq:pf1-1}. For $L=2$, it follows from \cref{lem:box_size_2}. Assume for $L \le m$, the result holds. Therefore, for $L = m+1$, we group the first $m$ layers as a single layer, resulting in a ``two'' layer equivalent network. Thus, $(\ubopt - \lbopt) / 2 = \left|\mW^{(m+1)} \Pi_{k=1}^m \mW^{(k)} \right| \bm{\epsilon} = \left| \Pi_{k=1}^L \mW^{(k)} \right| \bm{\epsilon}$. Similarly, by \cref{eq:pf1-1}, we can prove $(\ubbox - \lbbox) / 2 = \left(\left|\mW^{(m+1)} \right|\Pi_{k=1}^m \left|\mW^{(k)} \right|\right) \bm{\epsilon} = \left(\Pi_{k=1}^L \left|\mW^{(k)} \right|\right) \bm{\epsilon}$. The claim about center follows by induction similarly.
\end{proof}

\paragraph{Proof of \cref{lem:invariant}}

Here, we prove \cref{lem:invariant}, restated below for convenience.

\invariant*

\begin{proof} \label{pf:invariant}
    We prove the statement via comparing the box bounds. By \cref{lem:box_size_2}, we need $\left| \sum_{k=1}^{d_1} W^{(2)}_{i,k} W^{(1)}_{k,j} \right| = \sum_{k=1}^{d_1} |W^{(2)}_{i,k} W^{(1)}_{k,j}|$. The triangle inequality of absolute function says this holds if and only if $W^{(2)}_{i,k} W^{(1)}_{k,j} \ge 0$ for all $k$ or $W^{(2)}_{i,k} W^{(1)}_{k,j} \le 0$ for all $k$.
\end{proof}

\paragraph{Proof of \cref{thm:invariant-reg}}

Here, we prove \cref{thm:invariant-reg}, restated below for convenience.

\noninvariant*

\begin{proof}
Suppose the network is propagation invariant. By \cref{lem:invariant} and the entrywise nonzero assumption, each path product $W^{(2)}_{r,k}W^{(1)}_{k,s}$ has the sign of $(\mW^{(2)}\mW^{(1)})_{r,s}$ for $r\in\{i,i^\prime\}$ and $s\in\{j,j^\prime\}$.
For any fixed hidden index $k$, the product of these four path products is
\begin{equation*}
\bigl(W^{(2)}_{i,k}W^{(2)}_{i^\prime,k}W^{(1)}_{k,j}W^{(1)}_{k,j^\prime}\bigr)^2>0.
\end{equation*}
Thus the product of the four corresponding end-to-end entries is positive, contradicting the hypothesis.
\end{proof}

\paragraph{Proof of \cref{thm:tightnessinit}}

Here, we prove \cref{thm:tightnessinit}, restated below for convenience.

\initwidth*

\begin{proof}
    We first compute the radius of the layer-wise propagated box. From \cref{eq:pf1-2}, we get that for the $i$-th dimension,
    \begin{align*}
        \E[(u_i - l_i)/2]
         & = \E\left(\sum_{j=1}^{d_0} \epsilon_0[j] \left(\sum_{k=1}^{d_1} |W^{(2)}_{i,k} W^{(1)}_{k,j}|\right)\right)           \\
         & = \sum_{j=1}^{d_0} \epsilon_0[j] \left(\sum_{k=1}^{d_1} \E(|W^{(2)}_{i,k}|)\cdot \E(|W^{(1)}_{k,j}|)\right)           \\
         & = \sigma_1 \sigma_2\sum_{j=1}^{d_0} \epsilon_0[j] \left(\sum_{k=1}^{d_1} \E(|\mathcal{N}(0, 1)|)^2)\right).
    \end{align*}
    Since $\E(|\mathcal{N}(0, 1)|)=\sqrt{\frac{2}{\pi}}$ \footnote{https://en.wikipedia.org/wiki/Half-normal\_distribution}, we have
    \begin{equation}
        \E[(u_i - l_i)/2]  = \frac{2}{\pi}\sigma_1 \sigma_2 d_1  \|\epsilon_0\|_1 . \label{eq:3-1}
    \end{equation}
    Now we compute the radius of the tightest box. From \cref{eq:pf1-3}, we get that for the $i$-th dimension,
    \begin{align*}
        \E[(u_i^* - l_i^*)/2]
          = \E\left( \sum_{j=1}^{d_0} \epsilon_0[j]\left| \sum_{k=1}^{d_1} W^{(2)}_{i,k} W^{(1)}_{k,j} \right| \right)
          = \sigma_1 \sigma_2\sum_{j=1}^{d_0} \epsilon_0[j] \E \left( \left| \sum_{k=1}^{d_1} X_k Y_k \right| \right),
    \end{align*}
    where $X_k$ and $Y_k$ are i.i.d. standard Gaussian random variables. Using the law of total expectation, we have
    \begin{align*}
        \E \left( \left| \sum_{k=1}^{d_1} X_k Y_k \right| \right)
         & = \E \left( \E \left( \left| \sum_{k=1}^{d_1} X_k Y_k \right| \;\middle|\; Y_1,\ldots,Y_{d_1} \right) \right)           \\
         & = \E \left( \E \left(\left|\mathcal{N}(0, \sum_{k=1}^{d_1} Y_k^2) \right| \;\middle|\; Y_1,\ldots,Y_{d_1}\right)\right) \\
         & = \sqrt{\frac{2}{\pi}} \E \left( \sqrt{\sum_{k=1}^{d_1} Y_k^2} \right)                                   \\
         & = \sqrt{\frac{2}{\pi}} \E (\sqrt{\chi^2(d_1)}).
    \end{align*}
    Since $\E (\sqrt{\chi^2(d_1)}) = \sqrt{2} \Gamma(\frac{1}{2} (d_1+1)) / \Gamma(\frac{1}{2}d_1)$, \footnote{https://en.wikipedia.org/wiki/Chi\_distribution} we have
    \begin{equation}
        \E[(u_i^* - l_i^*)/2]  = \frac{2}{\sqrt{\pi}}\sigma_1 \sigma_2 \|\epsilon_0\|_1 \Gamma(\frac{1}{2} (d_1+1)) / \Gamma(\frac{1}{2}d_1). \label{eq:3-2}
    \end{equation}
    Combining \cref{eq:3-1} and \cref{eq:3-2}, we have:
    \begin{equation}
        \frac{\E[(u_i - l_i)/2]}{\E[(u_i^* - l_i^*)/2]} = \frac{d_1 \Gamma(\frac{1}{2}d_1)}{\sqrt{\pi}\Gamma(\frac{1}{2}(d_1+1))}. \label{eq:3-3}
    \end{equation}
    To see the asymptotic behavior, use $\Gamma(x+\alpha)/\Gamma(x) \sim x^\alpha$,\footnote{https://en.wikipedia.org/wiki/Gamma\_function\#Stirling's\_formula} we have
    \begin{equation}
        \frac{\E[(u_i - l_i)/2]}{\E[(u_i^* - l_i^*)/2]} \sim \sqrt{\frac{2d_1}{\pi}}.
    \end{equation}
    To establish the bounds on the minimum expected slackness, we use \cref{lem:property}.
\end{proof}

\begin{restatable}[]{lem}{growth}
    \label[lemma]{lem:property}
    Let $g(n):= \frac{n \Gamma(\frac{1}{2}n)}{\sqrt{\pi}\Gamma(\frac{1}{2}(n+1))}$. $g(n)$ is monotonically increasing for $n \ge 1$. Thus, for $n \ge 2$, $g(n) \ge g(2) > 1.27$.
\end{restatable}

\begin{proof}
    Using polygamma function $\psi^{(0)}(z) = \Gamma^\prime(z) / \Gamma(z)$,\footnote{https://en.wikipedia.org/wiki/Polygamma\_function} we have
    \begin{equation*}
        g^\prime(n) \propto 1+\frac{1}{2}n \left(\psi^{(0)}\left(\frac{1}{2}n\right) - \psi^{(0)}\left(\frac{1}{2}(n+1)\right)\right).
    \end{equation*}
    Using the fact that $\psi^{(0)}(z)$ is monotonically increasing for $z > 0$ and $\psi^{(0)}(z+1) = \psi^{(0)}(z) + \frac{1}{z}$, we have
    \begin{align*}
         & 1+\frac{1}{2}n \left(\psi^{(0)}\left(\frac{1}{2}n\right) - \psi^{(0)}\left(\frac{1}{2}(n+1)\right)\right) \\
         & > 1+\frac{1}{2}n \left(\psi^{(0)}\left(\frac{1}{2}n\right) - \psi^{(0)}\left(\frac{1}{2}n+1\right)\right) \\
         & = 1+\frac{1}{2}n \left(- \frac{2}{n}\right)                                                               \\
         & = 0.
    \end{align*}
    Therefore, $g^\prime(n)$ is strictly positive for $n \ge 1$, and thus $g(n)$ is monotonically increasing for $n \ge 1$.
\end{proof}

As a final comment, we visualize $g(n)$ in \cref{fig:gn}. As expected, $g(n)$ is monotonically increasing in the order of $O(\sqrt{n})$.

\begin{figure}
    \centering
    \includegraphics[width=.45\linewidth]{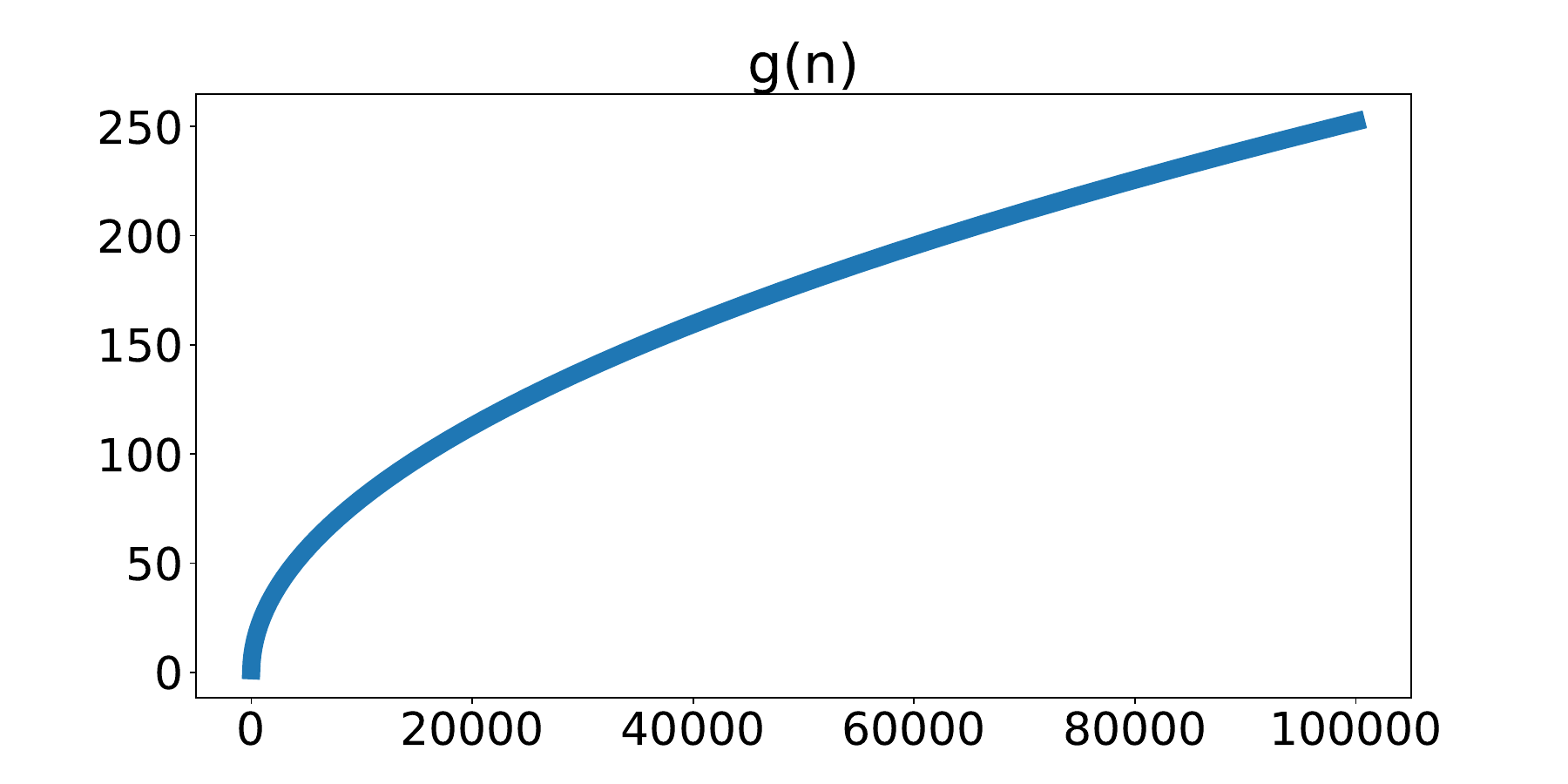}
    \includegraphics[width=.45\linewidth]{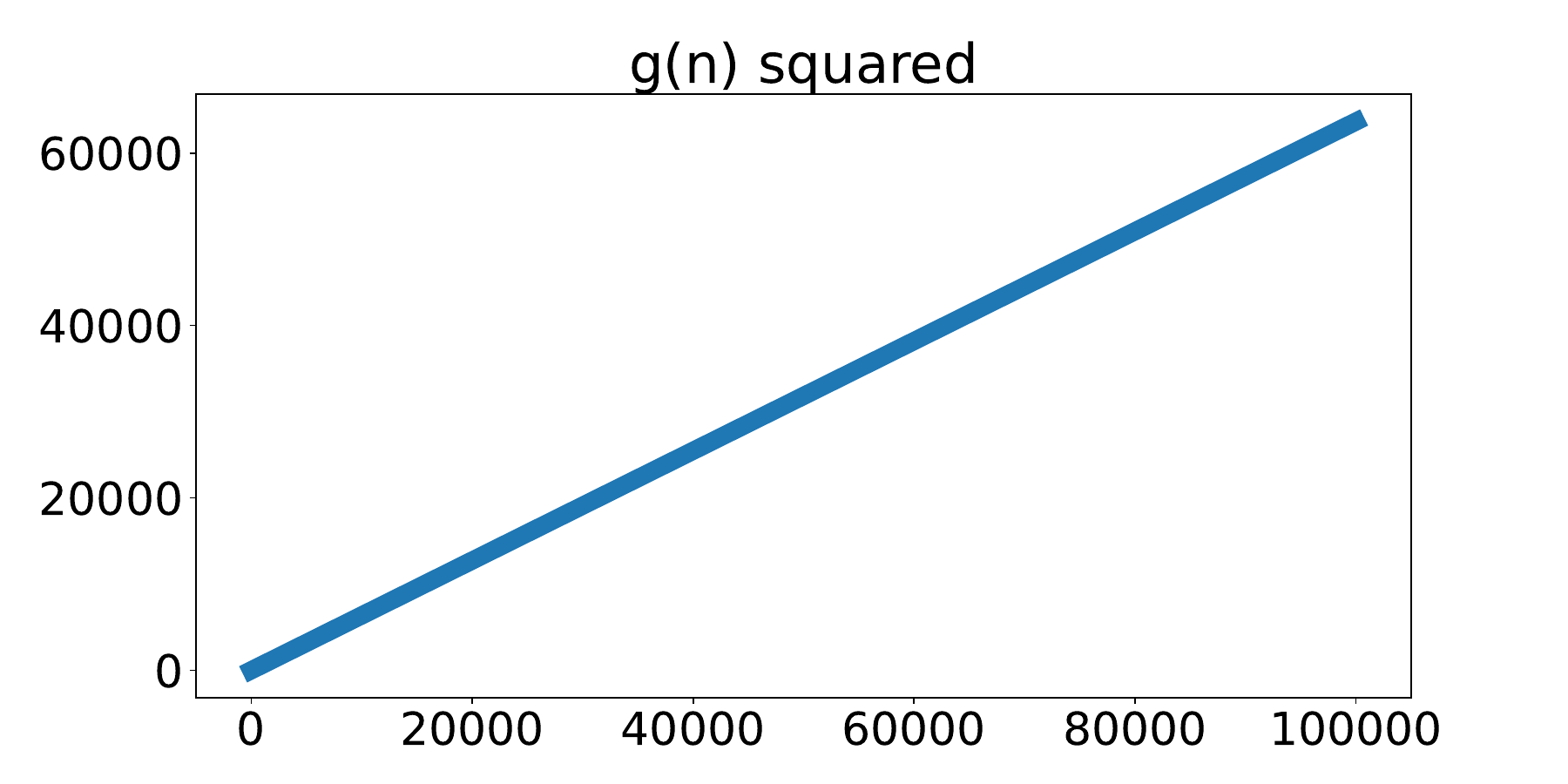}
    \caption{$g(n)$ and $g^2(n)$ visualized.} \label{fig:gn}
\end{figure}

\paragraph{Proof of \cref{cor:exp_growth}}

Here, we prove \cref{cor:exp_growth}, restated below for convenience.

\initdepth* 

\begin{proof}
Let $m=\lfloor L/2\rfloor$ and group consecutive layers into
$C_j=|\mW^{(2j)}\mW^{(2j-1)}|$ and $D_j=|\mW^{(2j)}||\mW^{(2j-1)}|$, $1\le j\le m$.
By \cref{thm:tightnessinit}, entrywise $\E C_j=\tau(d_{2j-1})\E D_j$.
For $L=2m$, the triangle inequality gives $|\mW^{(L)}\cdots\mW^{(1)}|\le C_m\cdots C_1$ entrywise.
The layer pairs are independent, so expectations of their matrix products factor:
\begin{align*}
\E|\mW^{(L)}\cdots\mW^{(1)}|
&\le (\E C_m)\cdots(\E C_1)\\
&=\left(\prod_{j=1}^m\tau(d_{2j-1})\right)(\E D_m)\cdots(\E D_1)\\
&=\left(\prod_{j=1}^m\tau(d_{2j-1})\right)\E\bigl[|\mW^{(L)}|\cdots|\mW^{(1)}|\bigr].
\end{align*}
For $L=2m+1$, the same argument includes the independent factor $|\mW^{(L)}|$ on the left of both grouped products.
Multiplying by a nonzero input-radius vector $\bm{\epsilon}\ge\bm{0}$ and dividing coordinatewise by the positive expected layer-wise radii gives $\tau\le\prod_{j=1}^m\tau(d_{2j-1})$.
Since $1/\tau(n)$ is increasing by \cref{lem:property} and $d_{2j-1}\ge d_\text{min}$, this is at most $\tau(d_\text{min})^m$.
\end{proof}

\paragraph{Proof of \cref{thm:radius_tightness}}

Here, we prove \cref{thm:radius_tightness}, restated below for convenience.

\ibptraining*

\begin{proof}
Let $v=\nabla_\theta\tau$,
$\gamma=\|v\|_2^2/\tau$, and
$\lambda=\gamma/(1+c)$. For a scalar quantity $a(\theta)$, define
$D_v a\coloneqq\langle\nabla_\theta a,v\rangle$, its rate of change as the parameters move in direction $v$.
For vectors and matrices, apply $D_v$ to each entry. For fixed $i,j$, set
$\alpha=\nabla_\theta\log\optW_{ij}$,
$\beta=\nabla_\theta\log\ibpW_{ij}$, and
$h=\alpha-\beta=v/\tau$. By the relative-gradient assumption,

\begin{equation*}
\|\alpha\|_2\le c\|\beta\|_2,
\qquad
\|\alpha\|_2\le\frac{c}{1-c}\|h\|_2,
\qquad
\|h\|_2\le(1+c)\|\beta\|_2.
\end{equation*}

Consequently,

\begin{equation*}
\begin{aligned}
-\langle\beta,h\rangle
&=\frac12\bigl(\|h\|_2^2+\|\beta\|_2^2-\|\alpha\|_2^2\bigr)\\
&\ge\frac12\left(1+\frac{1-c^2}{(1+c)^2}\right)\|h\|_2^2
=\frac{\|h\|_2^2}{1+c},
\end{aligned}
\end{equation*}

which gives

\begin{equation}
D_v\ibpW_{ij}
\le-\lambda\ibpW_{ij},
\qquad
|D_v\optW_{ij}|
\le\frac{c\gamma}{1-c}\optW_{ij}.
\label{eq:ibp_directional_bounds}
\end{equation}

Writing $\vf_\theta(\vx)=\mathbf{F}_\theta\vx$, we have $\optW=|\mathbf{F}_\theta|$ and $\mathbf{F}_{\theta,ij}\ne0$, so
$|D_v\mathbf{F}_{\theta,ij}|=|D_v\optW_{ij}|$. Thus
$\|\vx\|_\infty\le1$ and $2Mc(1+c)\le1-c$ imply

\begin{equation}
\|D_v\vf_\theta(\vx)\|_\infty
\le\frac{c\gamma M}{1-c},
\qquad
\max_iD_v f_{\theta,i}(\vx)
-\min_iD_v f_{\theta,i}(\vx)
\le\frac{2c\gamma M}{1-c}
\le\lambda.
\label{eq:ibp_logit_oscillation}
\end{equation}

Fix $\vx,y,\bm{\epsilon}$, and put
$\boldsymbol\rho=\ibpW\bm{\epsilon}$. Define

\begin{equation*}
z_i=f_{\theta,i}(\vx)-f_{\theta,y}(\vx),
\qquad
r_i=\begin{cases}
\rho_i+\rho_y,&i\ne y,\\
0,&i=y,
\end{cases}
\end{equation*}
\begin{equation*}
\ell_t=\log\sum_i e^{z_i+t r_i},
\qquad
p_{t,i}=\frac{e^{z_i+t r_i}}{\sum_j e^{z_j+t r_j}}.
\end{equation*}

Here $z_y=r_y=0$, and the robust-minus-standard loss is
$\Delta=\ell_1-\ell_0\ge0$. Let $d_i=D_v z_i$ and
$\operatorname{range}(d)=\max_i d_i-\min_i d_i$. From \cref{eq:ibp_directional_bounds,eq:ibp_logit_oscillation},

\begin{equation}
D_v r_i\le-\lambda r_i,
\qquad
\operatorname{range}(d)
=\max_iD_v f_{\theta,i}-\min_iD_v f_{\theta,i}
\le\lambda.
\label{eq:ibp_margin_bounds}
\end{equation}

Differentiating both the nominal logits and the propagated radii gives

\begin{equation*}
D_v\Delta=(p_1-p_0)^\top d+p_1^\top D_v r.
\end{equation*}

Since $r_i\ge0$,

\begin{equation*}
\begin{aligned}
(p_1-p_0)^\top d
&=\int_0^1\sum_i p_{t,i}r_i
\bigl(d_i-p_t^\top d\bigr)\,dt\\
&\le\operatorname{range}(d)\int_0^1p_t^\top r\,dt
=\operatorname{range}(d)\Delta.
\end{aligned}
\end{equation*}

Moreover, $\ell_t''=\operatorname{Var}_{p_t}(r)\ge0$, so
$\Delta=\int_0^1p_t^\top r\,dt\le p_1^\top r$. Therefore,

\begin{equation*}
D_v\Delta
\le\operatorname{range}(d)\Delta-\lambda p_1^\top r
\le\bigl(\operatorname{range}(d)-\lambda\bigr)\Delta
\le0.
\end{equation*}

Taking expectations yields

\begin{equation*}
\left\langle
\nabla_\theta\bigl(R(\bm{\epsilon})-R(\mathbf0)\bigr),
\nabla_\theta\tau
\right\rangle
=\E_{\vx,y}D_v\Delta\le0.
\end{equation*}

Bounded inputs and local differentiability justify interchanging expectation
and differentiation. Finally, if $c\le(1+4M)^{-1}$, then
$2Mc(1+c)\le4Mc\le1-c$, establishing the stated sufficient choice of
$c$.
\end{proof}

\paragraph{Proof of \cref{thm:ibp_reconstruction}}

Here, we prove \cref{thm:ibp_reconstruction}, restated below for convenience.

\embedding*

\begin{proof}
    Since linear box propagation maps the input center to the output center, both boxes are centered at $\hat{\vx}$. By \cref{lem:box_size_2}, $\bm{\delta}^\dagger=|U_k||U_k|^\top\epsilon\bm{1}$.
    With $V=|U_k|$, this gives
    \begin{equation*}
        \delta_i^\dagger=\epsilon\sum_{j=1}^k V_{ij}\|V_{:j}\|_1,
        \qquad \E(\delta_i^\dagger/\epsilon)=ck,
        \qquad c=\E(V_{ij}\|V_{:j}\|_1).
    \end{equation*}
    Each column of $U$ is uniform on the unit sphere and can be generated by normalizing a vector of independent standard Gaussians \citep{cook1957rational,marsaglia1972choosing}.
    For such a sphere vector $s$, symmetry gives $\E s_1^2=1/d$. When $d\ge2$, write $(s_1,s_2)=r(\cos\phi,\sin\phi)$, where $\phi$ is uniform and independent of $r$.
    Since $\E r^2=2/d$ and $\E|\cos\phi\sin\phi|=1/\pi$, we obtain
    \begin{equation*}
        c=\E\left(|s_1|\sum_{p=1}^d|s_p|\right)
        =\frac1d+(d-1)\E|s_1s_2|
        =\frac1d+\frac{2(d-1)}{\pi d}\longrightarrow\frac2\pi.
    \end{equation*}
    The same formula gives $c=1$ when $d=1$.

    To show how good the asymptotic result is, we run Monte-Carlo to get the estimation of $c$. As shown in the left of \cref{fig:MCMC_reconstruction}, the Monte-Carlo result is consistent to this theorem. In addition, it converges very quickly, \eg, stablizing at 0.64 when $d \ge 100$.
    \begin{figure}[tbp]
        \centering
        \includegraphics[width=.25\linewidth]{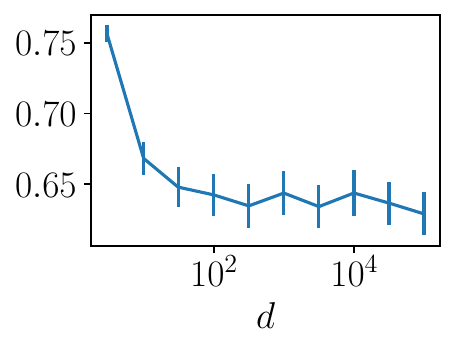}
        \includegraphics[width=.23\linewidth]{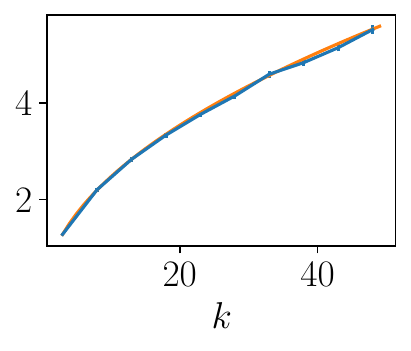}

        \caption{Monte-Carlo estimations of \cref{thm:ibp_reconstruction}. Result bases on 10000 samples for each $d$. Left: $c$ plotted against $d$ in log scale. Right: $\E(\bm{\delta}^*_i)$ plotted against $k$ for $d=2000$ (blue), together with the theoretical predictions (orange).} \label{fig:MCMC_reconstruction}
    \end{figure}

    For part (ii), write $P=U_kU_k^\top$, so $\bm{\delta}^*=|P|\epsilon\bm{1}$.
    By symmetry and $\E P_{ii}=k/d$,
    \begin{equation*}
        \E(\delta_i^*/\epsilon)=(d-1)\E|P_{ij}|+\frac{k}{d},\qquad j\ne i.
    \end{equation*}
    Condition on row $u$ of $U$ corresponding to index $i$. Row $j$ is uniform on the unit sphere in $u^\perp$.
    With $S_d=\sum_{p=1}^k u_p^2$, the projection of $(u_1,\ldots,u_k,0,\ldots,0)$ onto $u^\perp$ has squared length $S_d-S_d^2$.
    Thus $\E(|P_{ij}|\mid u)=q_{d-1}\sqrt{S_d(1-S_d)}$, where $q_m$ is the mean absolute first coordinate of a uniform sphere vector in $\mathbb{R}^m$.
    Independence of Gaussian radius and direction gives
    \begin{equation*}
        q_m=\frac{\E|\mathcal{N}(0,1)|}{\E\sqrt{\chi_m^2}}
        =\frac{\Gamma(m/2)}{\sqrt{\pi}\Gamma((m+1)/2)},
        \qquad \sqrt m\,q_m\longrightarrow\sqrt{\frac2\pi}.
    \end{equation*}
    For fixed $k$ and $d\to\infty$, the normalized Gaussian representation gives $dS_d\to A:=\sum_{p=1}^k Z_p^2$ almost surely, where the $Z_p$ are independent standard Gaussians.
    Consequently, $Y_d:=\sqrt{dS_d(1-S_d)}\to\sqrt A$ almost surely.
    Since $\E Y_d^2\le d\E S_d=k$, the sequence is uniformly integrable, so $\E Y_d\to\E\sqrt{\chi_k^2}$.
    Therefore,
    \begin{equation*}
        \E(\delta_i^*/\epsilon)
        =\frac{d-1}{\sqrt d}\,q_{d-1}\E Y_d+\frac{k}{d}
        \longrightarrow\sqrt{\frac2\pi}\,\E\sqrt{\chi_k^2}
        =\frac2{\sqrt\pi}\frac{\Gamma((k+1)/2)}{\Gamma(k/2)}.
    \end{equation*}
    The gamma-ratio asymptotic makes this limiting expression $\Theta(\sqrt k)$ as $k\to\infty$.
\end{proof}

The right of \cref{fig:MCMC_reconstruction} plots the Monte-Carlo estimations against our theoretical results. Clearly, this confirms our result.

\section{Experimental Details}\label{app:exp_details}

\subsection{Dataset}

We use the \mnist \citep{lecun2010mnist} and \cifar \citep{krizhevsky2009learning} datasets for our experiments. Both are open-source and freely available. 
For \mnist, we do not apply any preprocessing or data augmentation.
For \cifar, we normalize images with their mean and standard deviation and, during training, first apply 2-pixel zero padding and then random cropping to $32 \times 32$.

\subsection{Model Architecture}

We follow previous works \citep{ShiWZYH21,MuellerEFV22,MaoMFV2023} and use a 7-layer convolutional network \cnns in most experiments. We also use a simplified 3-layer convolutional network \cnnt in \cref{sec:training_method}. Details about them can be found in the released code.

\subsection{Training}

Following previous works \citep{MuellerEFV22,MaoMFV2023}, we use the initialization, warm-up regularization, and learning schedules introduced by \citet{ShiWZYH21}. Specifically, for \mnist, the first 20 epochs are used for $\epsilon$-scheduling, increasing $\epsilon$ smoothly from 0 to the target value. Then, we train an additional 50 epochs with two learning rate decays of 0.2 at epochs 50 and 60, respectively. For \cifar, we use 80 epochs for $\epsilon$-annealing, after training models with standard training for 1 epoch. We continue training for 80 further epochs with two learning rate decays of 0.2 at epochs 120 and 140, respectively. The initial learning rate is $5\times 10^{-3}$ and the gradients are clipped to an $L_2$ norm of at most $10.0$ before every step.

\subsection{Certification}

We apply \mnbab \citep{FerrariMJV22}, a sate-of-the-art \citep{BrixMBJL22,MuellerBBLJ22} verifier based on multi-neuron constraints \citep{MullerMSPV22,SinghGPV19B} and the branch-and-bound paradigm \citep{BunelLTTKK20}  to certify all models. \mnbab is a state-of-the-art complete certification method built on multi-neuron relaxations. For \cref{tb:sabr_wide}, we use the same hyperparameters for \mnbab as \citet{MuellerEFV22} and set the timeout to  1000 seconds. For other experiments, we use the same hyperparameters but reduce timeout to 200 seconds for efficiency reasons.

\section{Extended Empirical Evaluation}

\subsection{\staps -Training and Regularization Level} \label{app:STAPS_regularization}

\begin{figure}
    \centering
    \vspace{-10mm}
    \begin{subfigure}{.34\linewidth}
        \centering
        \includegraphics[width=\linewidth]{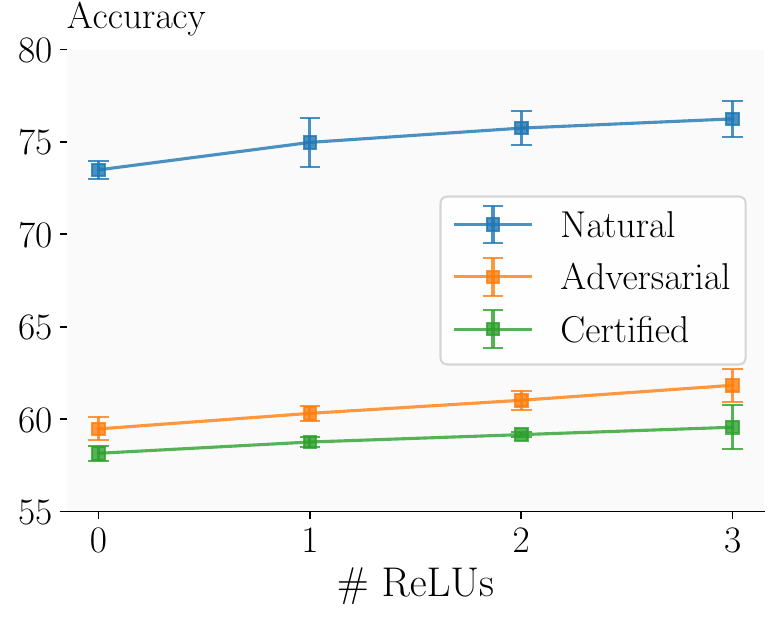}
    \end{subfigure}
    \hfil
    \begin{subfigure}{.34\linewidth}
        \centering
        \includegraphics[width=\linewidth]{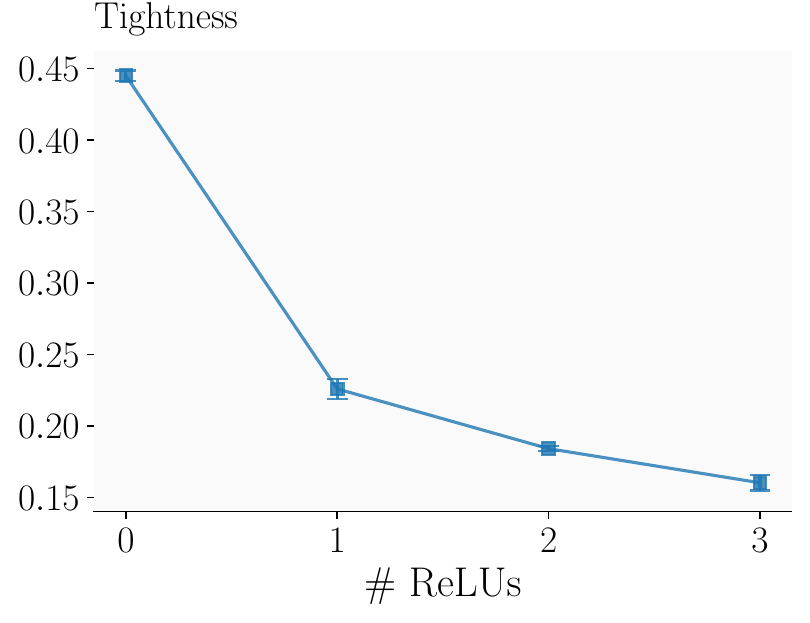}
    \end{subfigure}
    \caption{Accuracies and tightness of a \cnns for \cifar $\epsilon=\frac{2}{255}$ depending on regularization strength with \staps.}
    \label{fig:tightness_regularization_staps}
\end{figure}

To confirm our observations on the interaction of regularization level, accuracies, and propagation tightness from \cref{sec:training_method}, we extend our experiments to \staps \citep{MaoMFV2023}, an additional state-of-the-art certified training method beyond \sabr \citep{MuellerEFV22}. Recall that \staps combines \sabr with adversarial training as follows. The model is first (conceptually) split into a feature extractor and classifier. Then, during training \ibp is used to propagate the input region through the feature extractor yielding box bounds in the model's latent space. Then, adversarial training with \pgd is conducted over the classifier using these box bounds as input region. As \ibp leads to an over-approximation while \pgd leads to an under-approximation, \staps induces more regularization as fewer (ReLU) layers are included in the classifier.

We visualize the result of thus varying regularization levels by changing the number of ReLU layers in the classifier in \cref{fig:tightness_regularization_staps}. We observe very similar trends as for \sabr in \cref{fig:SABR_lambda}, although to a lesser extent, as $0$ ReLU layers in the classifier still recovers \sabr and not standard \ibp. Again, decreasing regularization (increasing the number of ReLU layers in the classifier) leads to reducing tightness and increasing standard and certified accuracies.

\subsection{ Tightness and Propagation Region Size} \label{app:tightness_prop_region_size}

We repeat the experiment illustrated in \cref{fig:SABR_lambda_tightness} for \cifar on \mnist using a \cnnt in \cref{fig:SABR_lambda_tightness_MNIST}. We again observe the propagation region size $\xi$ dominating the tightness (except for very large perturbation sizes of $\epsilon>0.2$), and smaller perturbation magnitudes leading to slightly larger tightness.

\begin{figure}
    \vspace{-10mm}
    \centering
    \includegraphics[width=0.34\linewidth]{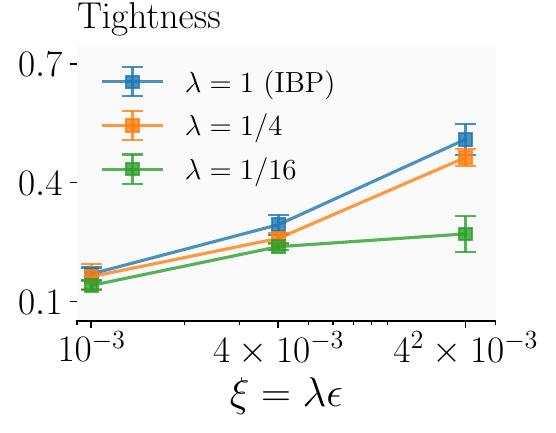}
    \vspace{-3.5mm}
    \caption{Tightness over propagation region size $\xi$ for \sabr and \mnist.}
    \label{fig:SABR_lambda_tightness_MNIST}
\end{figure}

\subsection{Tightness Approximation Error} \label{app:tightness_approximation_error}

\begin{wrapfigure}[14]{r}{0.39\linewidth}
    \vspace{-6mm}
    \centering
    \includegraphics[width=0.7\linewidth]{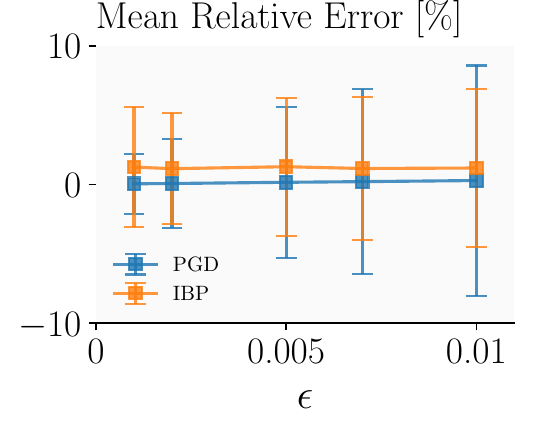}
    \vspace{-3mm}
    \caption{Mean relative error between local tightness (\cref{def:relu_tightness}) and true tightness computed with \milp for a \cnnt trained with \pgd or \ibp at $\epsilon=0.005$ on \cifar.}
    \label{fig:tightness_approximation_error_cifar}
\end{wrapfigure}

To investigate the approximation quality of our local tightness as defined in \cref{def:relu_tightness}, we compare it against the true tightness computed using MILP \citep{TjengXT19}. We confirm our results from \cref{fig:tightness_approximation_error} on \mnist on \cifar in \cref{fig:tightness_approximation_error_cifar}, where we again observe small approximation errors across a wide range of perturbation magnitudes. Interestingly, the effect of the chosen perturbation magnitude on the approximation error is less pronounced than on \mnist, remaining low even for large perturbation magnitudes ($\epsilon = 0.01 > 8/255$). While the approximation error remains below $0.3\%$ for a \pgd-trained net, our approximation exhibits a consistent bias for the \ibp-trained network, overestimating tightness by approximately $1.2\%$.

\subsection{Comparing Tightness to (Inverse) Robust Cross-Entropy Loss}
To investigate to what extent our novel tightness metric is complimentary to the (inverse) robust cross-entropy loss (see \cref{sec:background}) computed with \ibp, we repeat the key experiments confirming our theoretical insights with the inverse IBP-loss and observe significantly different, partially opposite trends.

\begin{wrapfigure}[18]{r}{.35\linewidth}
    \centering
    \vspace{-5mm}
    \begin{subfigure}{.95\linewidth}
        \centering
        \includegraphics[width=\linewidth]{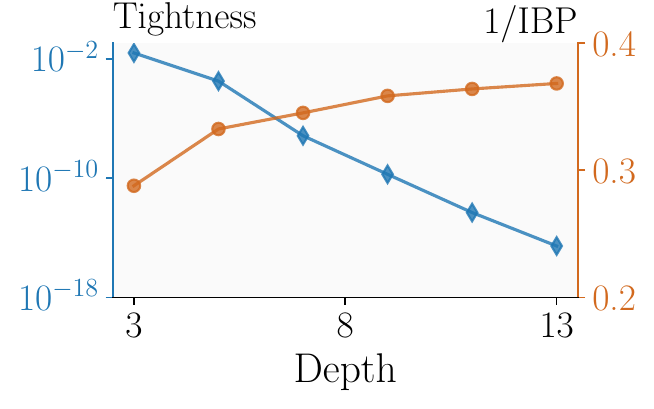}
    \end{subfigure}
    \begin{subfigure}{.95\linewidth}
        \centering
        \includegraphics[width=\linewidth]{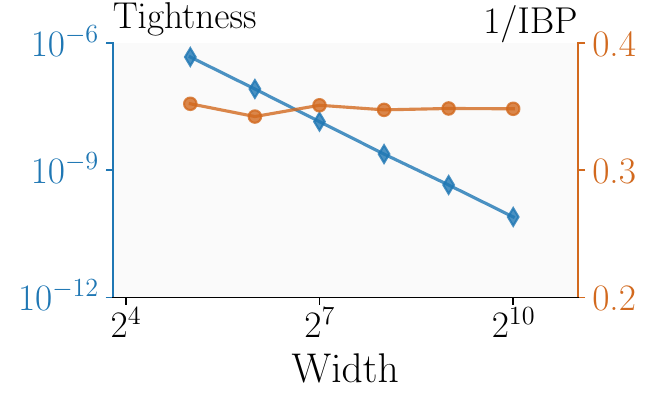}
    \end{subfigure}
    \vspace{-3.5mm}
    \caption{Tightness and inverse IBP loss at initialization depending on width and depth.}
    \label{fig:IBP_architecture_effect_ibp_loss}
\end{wrapfigure}
\paragraph{IBP Loss at Initialization}
We repeat our experiments on the dependence of tightness at initialization on network depth and width, illustrated in \cref{fig:IBP_architecture_effect}, and additionally report the inverse IBP loss in \cref{fig:IBP_architecture_effect_ibp_loss}. For all experiments, we use the initialization of \citet{ShiWZYH21} which has become the de-facto standard for IBP-based training methods. 
While we (theoretically and empirically) observe an exponential reduction in tightness with increasing depth, the inverse \ibp loss increases slightly. Similarly,
while we (theoretically and empirically) observe a polynomial reduction in tightness with increasing width, the inverse \ibp loss stays almost constant.
Note the logarithmic scale (and orders of magnitude larger changed) for tightness and the linear scale for the inverse \ibp loss.
This difference in trend is unsurprising as the custom initialization of \citet{ShiWZYH21} is designed to keep \ibp bound width constant over network depth and width.
We thus conclude that tightness and (inverse) \ibp loss yield fundamentally different results and insights when analyzing networks at initialization. 

\begin{figure}[t]
    \centering
    \vspace{-2mm}
    \begin{subfigure}{.31\linewidth}
        \centering
        \includegraphics[width=\linewidth]{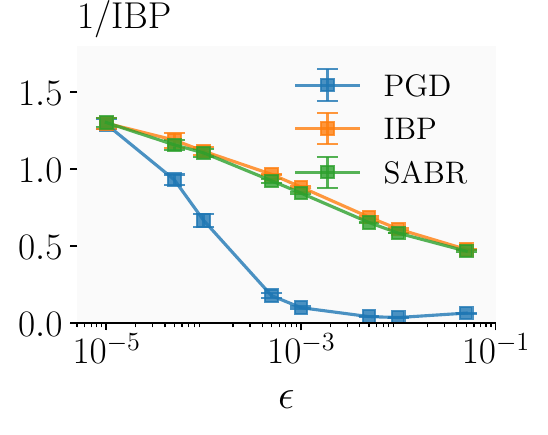}
        \vspace{-7mm}
        \caption{Evaluated at training $\epsilon$}
    \end{subfigure}
    \hfil
    \begin{subfigure}{.31\linewidth}
        \centering
        \includegraphics[width=\linewidth]{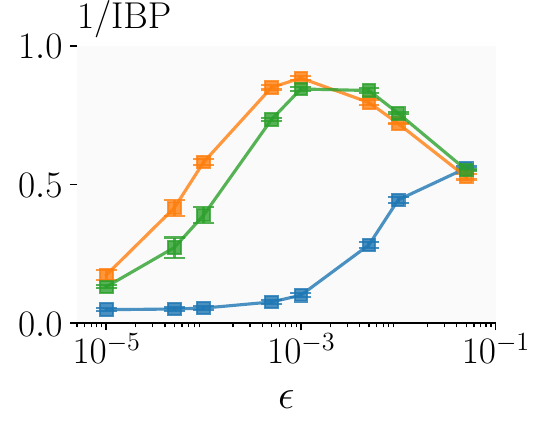}
        \vspace{-7mm}
        \caption{Evaluated at $\epsilon = 10^{-3}$}
    \end{subfigure}
    \hfil
    \begin{subfigure}{.31\linewidth}
        \centering
        \includegraphics[width=\linewidth]{figures/cnn3_method_PI}
        \vspace{-7mm}
        \caption{Evaluation $\epsilon$-independent}
    \end{subfigure}
    \vspace{-2mm}
    \caption{Tightness (right) and inverse IBP loss (left and center) after IBP training depending on training perturbation size $\epsilon$, evaluated at training $\epsilon$ (left) or constant $\epsilon=10^{-3}$ (center).}
    \label{fig:IBP_architecture_effect_ibp_loss_ibp}
\end{figure}

\paragraph{IBP Loss after Training}
We show the inverse IBP loss (left and center) and tightness (right) after IBP training depending on training perturbation size $\epsilon$, evaluated at training $\epsilon$ (left) or constant $\epsilon=10^{-3}$ (center) in \cref{fig:IBP_architecture_effect_ibp_loss_ibp}. 
We observe that inverse IBP loss, in contrast to tightness, is heavily dependent on the perturbation magnitude used for evaluation (compare left and center), making it poorly suited to analyze the effects of changing perturbation magnitude.
Further, when using the most natural perturbation magnitude, the $\epsilon$ used during training and certification (left), we observe completely different trends to tightness. 
For very small perturbation magnitudes, the inverse \ibp loss is very high, suggesting high (perturbation) robustness, but both inverse \ibp loss evaluated with a larger $\epsilon$ and tightness are low, showing that it neither permits precise analysis with \ibp nor is necessarily robust, again highlighting the difference between tightness and (inverse) \ibp loss.

\subsection{Tightness after \ibp Training} \label{app:ibp_training_tightness}
\begin{table}
    \centering
    \caption{Certified and standard accuracy depending on network width.} \label{tb:ibp_wide}
    \scalebox{0.8}{
        \begin{tabular}{@{}clcccccc@{}} \toprule
            Dataset                         & $\epsilon$                         & Method                 & Width       & Accuracy       & Certified    & Tightness  \\ \midrule
            \multirow{3}{*}{\mnist}         &  \multirow{3}{*}{$0.1$}           & \multirow{3}{*}{\ibp}  & $1\times$   & 85.70          & 67.71   &0.871       \\
                                            &                                    &                                                 & $2\times$   & 88.42          & 73.77       &0.857   \\       
                                            &                                    &                                                 & $4\times$   & \textbf{90.31}          & \textbf{79.89}  &0.803        \\
            \bottomrule
        \end{tabular}
    }
\end{table}
To confirm that wider models improve certified accuracy while slightly reducing tightness across network architectures, we also consider fully connected networks, which used to be the default in neural network verification \citep{SinghGPV19,SinghGMPV18}. We increase the width of a fully connected ReLU network with 6 hidden layers from 100 to 400 neurons and indeed observe a significant increase in certified accuracy (see \cref{tb:ibp_wide}).

}{}

\message{^^JLASTPAGE \thepage^^J}

\end{document}